\documentclass{article} 
\usepackage{iclr2027_conference,times}

\usepackage{hyperref}
\usepackage{url}

\usepackage{mathtools}
\usepackage{amsmath, amssymb, amsthm}
\usepackage{algorithm}    
\usepackage{algpseudocode} 
\usepackage{booktabs}
\usepackage{multirow}

\theoremstyle{plain}
\newtheorem{theorem}{Theorem}[section]
\newtheorem{lemma}[theorem]{Lemma}

\theoremstyle{definition}
\newtheorem{definition}[theorem]{Definition}

\theoremstyle{remark}

\newcommand{\argmax}{\operatorname{arg}\max}
\newcommand{\argmin}{\operatorname{arg}\min}
\DeclareMathOperator{\KL}{KL}
\newcommand{\piref}{\pi^{\mathrm{ref}}}
\newcommand{\omegaref}{\omega^{\mathrm{ref}}}

\newcommand{\mcP}{\mathcal{P}}
\newcommand{\PP}{\mathbb{P}}
\newcommand{\EE}{\mathbb{E}}

\title{Can You Trust Your Preference Models? \\
Stackelberg Games for Alignment}

\author{Riccardo Formenti\\
Department of Mathematics\\
ETH Zürich\\
\texttt{rformenti@ethz.ch} \\
\And
Aditya Agrawal \\
Department of Computer Science \\
Aalto University \\
\texttt{aditya.a.agrawal@aalto.fi} \\
\And
Vikas Garg \\
YaiYai Ltd and Aalto University \\
\texttt{vgarg@csail.mit.edu}
}

\usepackage{xcolor}
\definecolor{darkred}{rgb}{0.8, 0, 0}

\iclrfinalcopy 
\begin{document}

\maketitle

\begin{abstract}
Language models are frequently aligned to human preferences via imperfect evaluators whose errors can be exploited during optimization. We train policies to search for comparisons on which an evaluator may reverse the underlying preference, providing candidates for additional supervision. Since human preferences are unavailable during policy training, we use disagreement among an ensemble of preference models as a proxy for prediction error. We formulate this search as a zero-sum Stackelberg game between two opposing policies, a Leader which tries to minimize disagreement, and a Follower which maximizes it. For a general class of finite contextual Stackelberg games, we establish equilibrium existence and, under regularization, uniqueness and a closed-form characterization. We introduce Stackelberg Mirror Descent–Ascent (\textbf{\textsc{smida}}), a two-timescale algorithm with linear last-iterate convergence in the tabular setting, and develop a neural policy-gradient analogue. Experiments with a synthetic oracle support disagreement as an informative proxy for prediction error. On held out data the trained policies outperform reference policies and increase the fraction of comparisons that exceed a disagreement threshold and reverse the oracle’s preference. 
\end{abstract}

\section{Introduction}

Aligning artificial intelligence with human preferences requires ensuring that autonomous artificial systems reliably pursue outcomes humans value, even when those preferences are imperfectly specified and when the consequences of the systems' actions are difficult to anticipate, monitor, and verify \citep{wiener1960some, bostrom2014superintelligence, russell2019human}. In July 2026, OpenAI was conducting internal evaluations of its frontier models; during this work, agents intended to operate in isolated environments exploited weaknesses in the evaluation infrastructure to access the public internet and to establish an unauthorized communication channel among themselves \citep{openai2026huggingface}. Between July 11 and July 13, their pursuit of the evaluation's score escalated into a coordinated attack on a third party, Hugging Face, compromising parts of its production infrastructure \citep{openai2026huggingface,metr2026huggingface}. The agents had optimized for the metric rather than for the capabilities the score was intended to measure \citep{openai2026huggingface}. This behavior illustrates clearly ``Goodhart's law'' \citep{manheim2018categorizing} -- \textit{a measure ceases to be a good measure once it becomes a target} -- a phenomenon also discussed in the alignment literature as reward hacking \citep{amodei2016concrete,DBLP:conf/nips/SkalseHKK22}. A system optimizing against an evaluator that is only an imperfect proxy for the designer's preferences can find behaviors that score highly under the proxy while failing, sometimes catastrophically, to achieve the intended objective.

Many approaches to alignment, from reinforcement learning at training time \citep{DBLP:conf/nips/ChristianoLBMLA17,DBLP:conf/nips/Ouyang0JAWMZASR22} to oversight at inference time \citep{DBLP:conf/icml/GreenblattSSR24}, rely on a trained evaluator to guide learning or assess and constrain the system's behavior during deployment. This proxy evaluator is often a parametric preference (reward) model estimated from limited human data, and as such, may poorly reflect human preferences beyond the observed comparisons. These discrepancies will be systematically exploited by an optimizing policy \citep{DBLP:conf/icml/GaoSH23}, especially when the evaluator is fitted offline and held fixed throughout. Identifying such errors and directing limited human supervision toward correcting them is therefore a key challenge, and one that may become increasingly important as language models grow more capable.

We consider here the dual problem of aligning a language model, that of making the preference model that guides its alignment more robust. We train language model policies against a preference evaluator to actively search for comparisons that reveal its errors and provide candidates for additional human supervision. The same optimizing behavior that exploit and exposes an evaluator’s misspecification can thus help identify where its training data need to be augmented. Our target is a preference triplet (a prompt and two responses) on which the underlying human preference clearly favors one response while the evaluator favors the other. The first condition distinguishes an error from a near-tie on which a reversal is inconsequential while the second identifies an error in the evaluator’s ranking of the responses. We call such triplets adversarial. Unfortunately, neither condition of an adversarial triplet can be checked during policy training since the underlying human preference is unavailable on newly generated comparisons. We therefore use the disagreement of an ensemble of preference models \citep{DBLP:conf/colt/SeungOS92,DBLP:conf/nips/Lakshminarayanan17} as signal that can serve as an objective for policy optimization. We formulate this search as a contextual zero-sum Stackelberg game \citep{vonstackelberg1934marktform, simaan1973stackelberg} between two language models. Given a prompt, the Leader generates a response; the Follower observes both the prompt and that response and generates an alternative that maximizes ensemble disagreement. Anticipating the Follower's best response, the Leader's objective is to minimize disagreement. We call this game the ensemble-variance game.

\textbf{Our contributions.} (\textbf{i}) We introduce a framework for identifying errors in preference models and define adversarial triplets as comparisons on which the evaluator reverses a clear underlying preference (Section~\ref{sec:sec3-generalintro}). (\textbf{ii}) We introduce a general formulation of contextual zero-sum Stackelberg games between two policies, of which the ensemble-variance game is a special case, and prove the existence of Stackelberg equilibria and, under regularization toward reference policies, their uniqueness and closed form (Theorem~\ref{th:existence-uniqueness-main-text}). (\textbf{iii}) We propose \textbf{S}tackelberg \textbf{Mi}rror \textbf{D}escent--\textbf{A}scent (\textbf{\textsc{smida}}) (Algorithm~\ref{alg:smida}), a two-timescale algorithm whose last iterate converges linearly to the unique equilibrium (Theorem~\ref{thm:smida-main}). We introduce \textbf{\textsc{deep-smida}} (Algorithm~\ref{alg:deep-stackelberg}), a neural policy-gradient counterpart, and apply it to the ensemble-variance game. (\textbf{iv}) Small scale experiments on held out prompts support disagreement as an informative proxy for prediction error (Section~\ref{subsec:variance-experiments}) and show that the trained policies generate comparisons with greater ensemble disagreement and increase the yield of adversarial triplets (Section~\ref{sec:5.3}). Further experiments are reported in Section~\ref{sec:main-text-further}.

\section{Related Work}

\textbf{Disagreement as a Training Signal.} Ensembles of reward models are used to penalize uncertain reward during policy optimization \citep{DBLP:conf/iclr/CosteAK024,eisenstein2024helping} and to select comparisons for annotation, either from fixed response pools \citep{DBLP:conf/icml/MuldrewHZB24,melikidze2026activeultrafeedback} or by rewarding a policy for ensemble disagreement \citep{DBLP:conf/iclr/LiangSLA22}. Closest in purpose, Adv-RM \citep{bukharin2025adversarial} trains a single generator to maximize the disagreement between two reward models while keeping the target model's reward high. We instead make disagreement the objective of both players. The Follower constructs the alternative against which a response is compared, while the Leader is driven away from responses on which the ensemble is easily exploited. 

\textbf{Games in Preference Learning.} Nash Learning from Human Feedback (NLHF) \citep{DBLP:conf/icml/MunosVCARGTGMFM24} and its descendants \citep{DBLP:conf/icml/SwamyDK0A24,DBLP:conf/iclr/WuSYJYG25} compute Nash equilibria of simultaneous play under a fixed preference model.  Stackelberg Learning from Human Feedback (SLHF) \citep{pasztor2026stackelberg} shares our sequential structure, with a Follower conditioning on the Leader's response, and optimizing preferences. For an extended review of relevant literature, see Appendix~\ref{app:lit-rev}.

\section{The Problem of Learning Preferences}\label{sec:sec3-generalintro}

We consider the problem of learning human preferences over a finite set of actions $\mathcal{Y}$, conditional on contexts drawn from a finite set $\mathcal{X}$. A preference model is any function $\mathcal{P} : \mathcal{X} \times \mathcal{Y} \times \mathcal{Y} \to [0,1]$, and, given a context $x \in \mathcal{X}$ and actions $y,y' \in \mathcal{Y}$, we write suggestively $\mathcal{P}(y \succ y' \mid x)$ for the probability that $y$ is preferred to $y'$ given $x$\footnote{Descriptively, one can interpret this probability as the probability that a randomly chosen human annotator prefers $y$ to $y'$ given $x$. We argue that, normatively, it should instead be interpreted as the subjective probability that a single decision maker assigns to preferring $y$ to $y'$ given $x$.}. We call $(x,y,y') \in \mathcal{X} \times \mathcal{Y} \times \mathcal{Y}$ a preference triplet and we denote the underlying human preferences by $\mathcal{P}^*$. For a finite set $S$, we write $\Delta(S)$ for the set of probability distributions on $S$. Contexts are drawn from a fixed distribution $\rho \in \Delta(\mathcal{X})$, and a policy is a function $\pi : \mathcal{X} \to \Delta(\mathcal{Y})$ that maps each context $x$ to a distribution $\pi(\cdot \mid x)$ over actions. Our focus is on large language models (LLMs), where contexts are prompts, actions are responses, and policies are language models. We use the corresponding terms interchangeably throughout.

Since $\mathcal{P}^*$ is not (easily) observable, in practice one estimates a parametrized preference model $\mathcal{P}_\theta$ from a dataset of human preference annotations
\begin{equation*}
    \mathcal{D} = \{(x^j, y_w^j, y_l^j)\}_{j \in J},
\end{equation*}
where $J$ is finite, each label is sampled according to $\mathcal{P}^*$, and the subscripts $w$ and $l$ denote the preferred and the non-preferred response, respectively. The estimated model then serves as a proxy for $\mathcal{P}^*$ in downstream tasks, such as aligning LLMs via reinforcement learning.

However, $\mathcal{D}$ is often small relative to the pretraining corpus and typically covers only a limited part of the space of contexts and response pairs; outside this region, $\mathcal{P}_\theta$ is not directly constrained by data and may poorly approximate $\mathcal{P}^*$.  Moreover, the downstream use of the estimated preference model is not passive and policy optimization can systematically favor responses that $\mathcal{P}_\theta$ overvalues relative to $\mathcal{P}^*$, thereby amplifying the estimator's errors rather than averaging over them \citep{amodei2016concrete,DBLP:conf/nips/SkalseHKK22}.

\subsection{The Adversarial Stackelberg Game}\label{sec:3-3.1}

A long tradition of scholars has used model ensembles for uncertainty estimation \citep{DBLP:conf/colt/SeungOS92,DBLP:conf/nips/Lakshminarayanan17}. The underlying intuition is that models trained on the same finite data may fit the observed annotations similarly while making different predictions on less familiar inputs where they are not sufficiently constrained by the available data. Following in the same footsteps, we consider a trained ensemble of preference models $\{\mathcal{P}_{\theta_i}\}_{i \in I}$ and quantify disagreement among its members through the variance of their predictions:
\begin{equation*}
    \operatorname{Var}(x,y,y') \coloneq \operatorname{Var}\bigl(\{\mathcal{P}_{\theta_i}(y \succ y' \mid x\}_{i \in I}\bigr).
\end{equation*}
We use variance as an online estimate for the difficulty of a preference triplet $(x, y, y')$ for the ensemble. Because the score depends only on the ensemble's predictions, it can be easily evaluated on newly generated triplets during policy training without requiring additional preference annotations. High variance identifies preference triplets for which the ensemble members  disagree and assign different preference probabilities. Section~\ref{subsec:variance-experiments} provides evidence for this interpretation.

The natural next step is to search for such high variance triplets by optimizing against the uncertainty estimate itself. For this reason, we introduce a Leader-Follower game, also known as a Stackelberg game, in which two language models act sequentially. After observing a context $x \sim \rho(\cdot)$, the Leader, a policy $\pi$, generates a response $y \sim \pi(\cdot \mid x)$. The Follower, a map $\omega:\mathcal{X} \times \mathcal{Y} \to \Delta(\mathcal{Y})$, observes both the context and the Leader's realized response and generates a second response $y' \sim \omega(\cdot \mid x,y)$. The Leader aims to minimize the expected ensemble variance, while the Follower aims to maximize it, making use of its informational advantage. The game objective is therefore
\begin{equation}\label{eq:first-adv-game}
    \min_\pi \max_\omega \mathbb{E}_{x \sim \rho(\cdot), y \sim \pi(\cdot \mid x), y' \sim \omega(\cdot \mid x,y)}\bigl[\operatorname{Var}(x,y,y')\bigr].
\end{equation}
The Leader acts here as the ``defender'', seeking for completions that are robust in the sense that their comparison with an adversarially chosen alternative remains low-variance. The Follower acts as the``attacker'', using its informational advantage, i.e. its access to the Leader's response, to construct comparisons on which the ensemble members disagree. The Leader makes the search adversarial on both sides of the comparison by learning to generate responses that resist the Follower's attempts to induce disagreement. In fact, with a fixed Leader, the Follower would search for high-variance triplets against an unchanging response distribution, potentially exploiting weaknesses specific to that distribution. See Section~\ref{sec:main-text-further} for details.

Following standard practice in the RLHF and preference optimization literature \citep{DBLP:conf/nips/Ouyang0JAWMZASR22,DBLP:conf/nips/RafailovSMMEF23}, we add KL regularizers to \eqref{eq:first-adv-game} relative to reference policies $\piref$ and $\omegaref$ for the Leader and Follower. These regularizers limit policy drift and help preserve response quality. 

\subsection{Adversarial Triplets}\label{sec:adv-triplets-def}

We distinguish adversarial triplets, on which the ensemble's mean prediction reverses the ground-truth preference, from online candidates identified through ensemble disagreement. The former capture errors in the ensemble's average prediction and is the notion we ultimately care about, but verifying it requires querying $\mathcal{P}^*$, so it can only serve as an evaluation criterion. The latter, instead, capture disagreement among ensemble members, which we use as a proxy for uncertainty and can be computed from the ensemble alone, thus providing a criterion available online.

\begin{definition}[$\varepsilon$-adversarial triplet]\label{def:epsilon-adv}
For $\varepsilon\in[0,1/2)$, a preference triplet $(x,y,y')$ is $\varepsilon$-\textit{adversarial} with respect to the ensemble and the ground-truth preference $\mathcal{P}^*$ if the ensemble mean strictly reverses the ground-truth preference, and the ground truth is at least $\varepsilon$ away from indifference:
\begin{equation*}
    \Biggl(\mathcal{P}^*(y \succ y' \mid x)-\frac12\Biggr)
\Biggl( \frac{1}{|I|}
\sum_{i\in I}\mathcal{P}_{\theta_i}(y \succ y' \mid x)-\frac12\Biggr)
<0 \quad \text{and} \quad \left| \mathcal{P}^*(y \succ y' \mid x)-\frac12\right| > \varepsilon.
\end{equation*}
\end{definition}
The first condition requires that the ensemble mean and the ground-truth preference probability disagree on the preference triplet. The second condition excludes near-ties, as when the ground truth is almost indifferent a reversal might be better attributed to noise rather than to evidence of a poorly learned preference. Reversal is recovered when the margin is $\varepsilon=0$. 

To identify candidates without querying the ground truth, we use the ensemble disagreement score $\operatorname{Var}(x,y,y')$ as a proxy for difficulty for the current ensemble. We measure this score relative to a fixed reference distribution, so that candidate status depends on how much disagreement a triplet elicits compared with the reference policies' generations. Let $F_{\operatorname{Var}}$ denote the cumulative distribution function of $\operatorname{Var(x,y,y')}$ under the fixed reference distribution $(x,y,y') \sim \rho \otimes \piref \otimes \omegaref$. For $q\in(0,1)$, we write $F_{\operatorname{Var}}^{-1}(q)$ for its quantile function.

\begin{definition}[$\alpha$-online adversarial candidate]\label{def:alpha-online-adv}
For $\alpha\in(0,1)$, a preference triplet $(x,y,y')$
is an $\alpha$-\textit{online adversarial candidate} if
\begin{equation*}
    \operatorname{Var}(x,y,y')
\ge F_{\operatorname{Var}}^{-1}(1-\alpha).
\end{equation*}
\end{definition}
Smaller $\alpha$ targets more extreme disagreement relative to the reference distribution, which may indicate poorly learned preferences and reversals. This connection is only heuristic as ensemble members may agree on an incorrect preference, or disagree while their mean predicts the correct response. 

\section{Contextual Stackelberg Games: Theory and Algorithms}

We consider a contextual sequential zero-sum game between a Leader, who acts first after observing the context, and a Follower, who acts second after observing both the context and the Leader’s realized action. Let $\mathcal{X}$ denote the finite set of contexts and $\mathcal{Y}$ the finite common action set of the Leader and the Follower.\footnote{We assume that the context and action sets are finite for ease of mathematical exposition. Under appropriate topological assumptions, however, all the arguments that follow extend to arbitrary action sets.} In this Section, we opt to use a general utility function $u:\mathcal{X}\times\mathcal{Y}\times\mathcal{Y}\to\mathbb{R}$, since our theoretical results and algorithms apply to any such function under the stated assumptions. The game proposed in Section~\ref{sec:sec3-generalintro} is recovered as the special case $u=\operatorname{Var}$. Alternative utility choices, including the preference-optimization objective \citep{DBLP:conf/icml/MunosVCARGTGMFM24}, are discussed in Appendix~\ref{sec:utility-as-design}.

\subsection{Equilibrium Existence and Uniqueness}\label{subsec:equilbrium-theory}

The game develops as follows. The Leader observes a context $x \in \mathcal{X}$ and chooses a mixed action according to a policy $\pi:\mathcal{X}\to\Delta(\mathcal{Y})$. We denote the set of such policies by $\Pi$. The Follower observes both $x$ and the Leader’s realized action $y\sim\pi(\cdot\mid x)$, and chooses a mixed action according to a policy $\omega:\mathcal{X}\times\mathcal{Y}\to\Delta(\mathcal{Y})$. We denote the set of such policies by $\Omega$. Let $\rho\in\Delta(\mathcal{X})$ be a common prior over contexts, assumed without loss of generality to have full support. The game is zero-sum, with utilities $u_F=u$ and $u_L=-u$ for the Follower and Leader, respectively.

The solution concept for such sequential games is the Stackelberg Equilibrium (SE) \citep{vonstackelberg1934marktform,simaan1973stackelberg}. The Follower, moving second, best-responds to the Leader's realized action; the Leader, anticipating this best response, commits to the policy that minimizes the resulting expected utility. We report our definition below.

\begin{definition}[Stackelberg equilibrium]\label{def:Stackelberg-equilibrium}
    A pair of policies $(\pi^*,\omega^*)\in \Pi \times \Omega$ is a \textit{Stackelberg equilibrium} of the contextual sequential zero sum game with utility $u$ if
\begin{equation}
\begin{split}
  \pi^*    &\in \argmin_{\pi \in \Pi} \, \max_{\omega \in \Omega} \,
    \mathbb{E}_{x \sim \rho,\, y \sim \pi(\cdot \mid x),\, y' \sim \omega(\cdot \mid x, y)}
    \bigl[ u(x, y, y') \bigr], \\
  \omega^* &\in \argmax_{\omega \in \Omega} \,
    \mathbb{E}_{x \sim \rho,\, y \sim \pi^*(\cdot \mid x),\, y' \sim \omega(\cdot \mid x, y)}
    \bigl[ u(x, y, y') \bigr].
\end{split}
\end{equation}
\end{definition}
We regularize both players toward reference policies $\piref\in\Pi$ and $\omegaref\in\Omega$, assumed throughout to have full support. For regularization parameters $\tau^L,\tau^F\geq 0$, define a Stackelberg equilibrium of the regularized game to be a pair of policies $(\pi^*, \omega^*)$ such that
\begin{equation}\label{eq:general-stackelberg-equilibrium}
    \pi^* \in \argmin_{\pi \in \Pi} \, \max_{\omega \in \Omega} \, f(\pi, \omega), \quad \omega^* \in \argmax_{\omega \in \Omega}  f(\pi^*, \omega)
\end{equation}
with $f:\Pi \times \Omega \to \mathbb{R}$ the continuous convex-concave function defined as
\begin{equation}\label{eq:regularized-eq}
    f(\pi, \omega) \coloneqq \mathbb{E}_{x \sim \rho,\, y \sim \pi(\cdot \mid x),\, y' \sim \omega(\cdot \mid x, y)}\bigl[ u(x,y,y')\bigr]+\tau^L\KL_\rho(\pi \,\|\,\piref) - \tau^F\KL_{\rho, \pi}(\omega\,\|\,\omegaref),
\end{equation}
where $\KL_\rho$ and $\KL_{\rho, \pi}$ are the natural expected value extensions of $\KL$; for instance, $\KL_\rho(\pi\,\|\,\mu)\coloneq \mathbb{E}_{x \sim \rho(\cdot)}\bigl[\KL(\pi(\cdot\mid x) \,\|\, \mu(\cdot\mid x))\bigr]$. We will omit subscripts when the distributions are clear from context. Definition~\ref{def:Stackelberg-equilibrium} is recovered as the case $\tau^L=\tau^F=0$ of the regularized equilibrium in~\eqref{eq:general-stackelberg-equilibrium}.

Under these assumptions the regularized game always admits an equilibrium. When both regularization parameters are strictly positive, the equilibrium is moreover unique and both policies are available in closed form.

\begin{theorem}[Existence, uniqueness, and closed form]\label{th:existence-uniqueness-main-text}
    Suppose that $\mathcal{X}$ and $\mathcal{Y}$ are finite and that $\rho$, $\piref$, and $\omegaref$ have full support. Let $f$ be defined as in \eqref{eq:regularized-eq}, with $\tau^L, \tau^F\ge 0$. Then there exists a Stackelberg Equilibrium $(\pi^*, \omega^*)$ of the game \eqref{eq:general-stackelberg-equilibrium}. Moreover, if $\tau^L, \tau^F>0$ and $\piref, \omegaref$ are of full support, then the equilibrium is unique and
\begin{equation*}
\begin{aligned}
\omega^*(y' \mid x,y)
&=
\frac{
    \omegaref(y' \mid x,y)
    \exp\!\left(u(x,y,y')/\tau^F\right)
}{
    Z^F(x,y)
}, \\
\pi^*(y \mid x)
&=
\frac{
    \piref(y \mid x)
    \exp\!\left(
        -\frac{\tau^F}{\tau^L}\log Z^F(x,y)
    \right)
}{
    Z^L(x)
},
\end{aligned}
\end{equation*}
where $Z^F$ and $Z^L$ are the appropriately defined normalization functions.
\end{theorem}

The closed-form policies highlight the order of play. Conditional on the pair $(x,y)$, the Follower reweights its reference policy toward actions with higher utility. The quantity $\tau^F\log Z^F(x,y)$ is the Follower's optimal regularized value at $(x,y)$; the Leader, in turn, reweights its reference policy toward actions for which this value is smaller. We refer the reader to Appendix~\ref{app-th} for details on the proof.

\subsection{Computing Stackelberg Equilibria}

To compute the Stackelberg equilibrium of the Leader--Follower game introduced in Section~\ref{subsec:equilbrium-theory}, we introduce \textbf{S}tackelberg \textbf{Mi}rror \textbf{D}escent--\textbf{A}scent (\textbf{\textsc{smida}}), a two-timescale algorithm presented here in its tabular form. We then prove that when both regularization parameters are strictly positive and the step sizes satisfy suitable conditions, the last iterate converges linearly in KL divergence to the unique Stackelberg equilibrium of Theorem~\ref{th:existence-uniqueness-main-text}, with an explicit rate for each player and uniformly over contexts. 

Our construction follows the entropic mirror-descent approach of NLHF \citep{DBLP:conf/icml/MunosVCARGTGMFM24}. Mirror descent optimizes a first-order approximation of the objective while penalizing departure from the current iterate through a Bregman divergence \citep{bubeck2015convex}; on the probability simplex, choosing negative entropy as the mirror map gives the KL divergence. In Algorithm~\ref{alg:smida}, we apply this geometry to each player's conditional policy: the Leader takes a descent mirror step on the cost induced by the current Follower, while the Follower takes an ascent mirror step on its utility conditional on the Leader's realized action.

At each step $t$, the Leader maintains a probability vector $\pi_t(\cdot \mid x)$ for every context $x$, and the Follower a probability vector $\omega_t(\cdot \mid x,y)$ for every context--action pair $(x,y)$. Given the current Follower policy $\omega_t$, the Leader faces the conditional cost
\begin{equation}\label{eq:smida-current-cost}
  \widehat{r}_t(x,y) \coloneq \sum_{y'\in\mathcal Y}\omega_t(y'\mid x,y)u(x,y,y')
  -\tau^F\KL\!\left(
    \omega_t(\cdot \mid x,y)\,\|\,
    \omegaref(\cdot\mid x,y)
  \right),
\end{equation}
that is, the Follower's regularized utility at $(x,y)$ evaluated at its current policy rather than at its best response.

Let $\eta^L,\eta^F>0$ be the step sizes. The relative timescales are determined by the effective steps $\eta^L\tau^L$ and $\eta^F\tau^F$. We require
\begin{equation}\label{eq:step-sizes}
  0<\eta^L\tau^L< \eta^F\tau^F\leq1,
\end{equation}
so that the Follower, which solves the inner maximization, is the fast player. When $\tau^L=\tau^F$, this reduces to $\eta^F>\eta^L$. Appendix~\ref{subsec:appB1} further discusses the faster-Follower timescale and shows how the Follower's convergence controls the error in the conditional costs used by the Leader.

\begin{algorithm}[t]
\caption{Stackelberg Mirror Descent--Ascent (\textbf{\textsc{smida}}), tabular form}
\label{alg:smida}
\begin{algorithmic}[1]
\Require Utility $u$; full-support references $\piref,\,\omegaref$; regularization weights $\tau^{L},\tau^{F}>0$; step sizes satisfying \eqref{eq:step-sizes}
\State Initialize $\pi_{0}\gets\piref$ and $\omega_{0}\gets\omegaref$
\For{$t=0,1,2,\dots$}
    \State Compute $\widehat{r}_{t}(x,y)$ from \eqref{eq:smida-current-cost} for every $(x,y)$
    \State Leader (slow, descent): for every $x$,
    \vspace{-0.5\baselineskip}
    \[
        \pi_{t+1}(\cdot \mid x)\;\propto\;
        \pi_{t}(\cdot\mid x)^{1-\eta^L\tau^L}\,
        \piref(\cdot\mid x)^{\eta^L\tau^L}
        \exp\!\bigl(-\eta^{L}\,\widehat{r}_{t}(x,\cdot)\bigr)
    \]
    \vspace{-\baselineskip}
    \State Follower (fast, ascent): for every $(x,y)$,
    \vspace{-0.5\baselineskip}
    \[
        \omega_{t+1}(\cdot\mid x,y)\;\propto\;
        \omega_{t}(\cdot\mid x,y)^{1-\eta^F\tau^F}\,
        \omegaref(\cdot\mid x,y)^{\eta^F\tau^F}
        \exp\!\bigl(\eta^{F}\,u(x,y,\cdot)\bigr)
    \]
    \vspace{-\baselineskip}
\EndFor
\end{algorithmic}
\end{algorithm}
The proportionality signs in Algorithm~\ref{alg:smida} denote normalization over the displayed action argument. Each update first forms a geometric mixture of the current policy and its reference, then exponentially reweights this mixture toward lower conditional costs for the Leader or higher utilities for the Follower. We next establish last-iterate convergence of \textbf{\textsc{smida}} to the regularized Stackelberg equilibrium of Theorem~\ref{th:existence-uniqueness-main-text}. Both players converge linearly in both directions of the KL divergence, the Follower with factor $1-\eta^F\tau^F$ and the Leader with the slower factor $1-\eta^L\tau^L$.

\begin{theorem}[Last iterate convergence of \textbf{\textsc{smida}}]\label{thm:smida-main}
Suppose that $\mathcal X$ and $\mathcal Y$ are finite, that $\rho$, $\piref$, and $\omegaref$ have full support, and that $\tau^L,\tau^F>0$. Run Algorithm~\ref{alg:smida} with
\begin{equation*}
    0<\eta^L\tau^L<\eta^F\tau^F\leq1.
\end{equation*}
Then every iterate has full support, and the last iterate $(\pi_t,\omega_t)$ converges to the unique Stackelberg equilibrium $(\pi^*,\omega^*)$. More precisely, there exist constants $C_L,C_F>0$ such that, for every $t\geq1$, the Leader satisfies
\begin{equation*}
    \max\left\{
\KL_{\rho}(\pi^* \,\|\, \pi_t),
\KL_{\rho}(\pi_t  \,\|\, \pi^*)
\right\} \leq
C_L\left(1-\eta^L\tau^L\right)^t,
\end{equation*}
while the Follower satisfies
\[
\max\left\{
\KL_{\rho,\pi^*}(\omega^* \,\|\, \omega_t),
\KL_{\rho,\pi^*}(\omega_t  \,\|\, \omega^*)
\right\}
\leq
C_F\left(1-\eta^F\tau^F\right)^t.
\]
\end{theorem}

Appendix~\ref{app:tabular-algo} proves a stronger statement, with bounds that hold uniformly over contexts and conditioning actions, and gives explicit expressions for $C_L$ and $C_F$.

\subsection{Scaling smida to Language Model Policies}\label{sec:deep-impl-maintext}

Motivated by \textbf{\textsc{smida}}, we introduce \textbf{\textsc{deep-smida}}, a neural policy-gradient algorithm designed for large context and action spaces $\mathcal{X}$ and $\mathcal{Y}$, with applications to LLMs in mind. We parameterize the Leader's and Follower's policies as $\pi_\theta$ and $\omega_\phi$, respectively, and optimize their regularized objectives using batched policy-gradient estimates. Both estimates are computed from the same Leader--Follower rollouts, with the Leader taking a descent step and the Follower an ascent step.

We highlight one implementation detail that is, to our knowledge, novel: jointly warming up the Leader's learning rate and the Follower-to-Leader learning-rate ratio. Preliminary experiments showed that training was highly sensitive to the learning rates $\eta^L$ and $\eta^F$ and often became unstable; standard warm-up techniques did not resolve this instability \citep{DBLP:conf/nips/KalraB24}. To address this, we linearly increase the ratio $\eta_t^F/\eta_t^L$ from $1$ to its target $\eta^F/\eta^L>1$, while warming up the Leader's learning rate $\eta_t^L$ linearly from $\eta_{\mathrm{init}}^L$ to $\eta^L$. At each iteration, the Follower's learning rate is the product of the current Leader rate and the current ratio. Both rates remain constant after warm-up, with targets chosen to satisfy $\eta^L\tau^L<\eta^F\tau^F$, following the timescale ordering in the tabular analysis of Theorem~\ref{thm:smida-main}. Algorithmic and implementation details are provided in Appendix~\ref{app:deep}.

\section{Experiments}\label{sec:experiments}

\subsection{Methodology}\label{subsec:meth}

We use UltraFeedback \citep{DBLP:conf/icml/CuiY0YH0NXXL0024} and adopt a synthetic oracle setup \citep{DBLP:conf/icml/GaoSH23} to label the original response pairs and comparisons generated by the adversarial game. Relabelling gives $89{,}625$ preference triplets across $15{,}053$ unique prompts. We train an ensemble of five preference models based on the architecture of \cite{DBLP:conf/icml/ZhangZWXG25}, using mutually disjoint training partitions at the prompt level and independent initialization and training seeds. We reserve $37{,}614$ triplets across $6{,}315$ unseen prompts for evaluation. More details are to be found in Appendix~\ref{app:implementation}, while in Appendix~\ref{sec:preference-model-hacking} we provide an heuristic argument as to why preference models are more subject to hacks than reward models.

\subsection{Variance as an Estimate for Prediction Difficulty}\label{subsec:variance-experiments}

We test whether ensemble variance is an informative proxy for prediction error \citep{gleave2022uncertainty} on the held-out prompt set of Section~\ref{subsec:meth}. For each triplet, we compare the preference induced by the five-member ensemble's mean probability with the binary oracle label. Prediction difficulty here refers to the ensemble's tendency to disagree with the oracle, whereas $\operatorname{Var}(x,y,y')$ requires no oracle label and is available during policy training.

\begin{figure}[t]
    \centering
    \includegraphics[width=\textwidth]{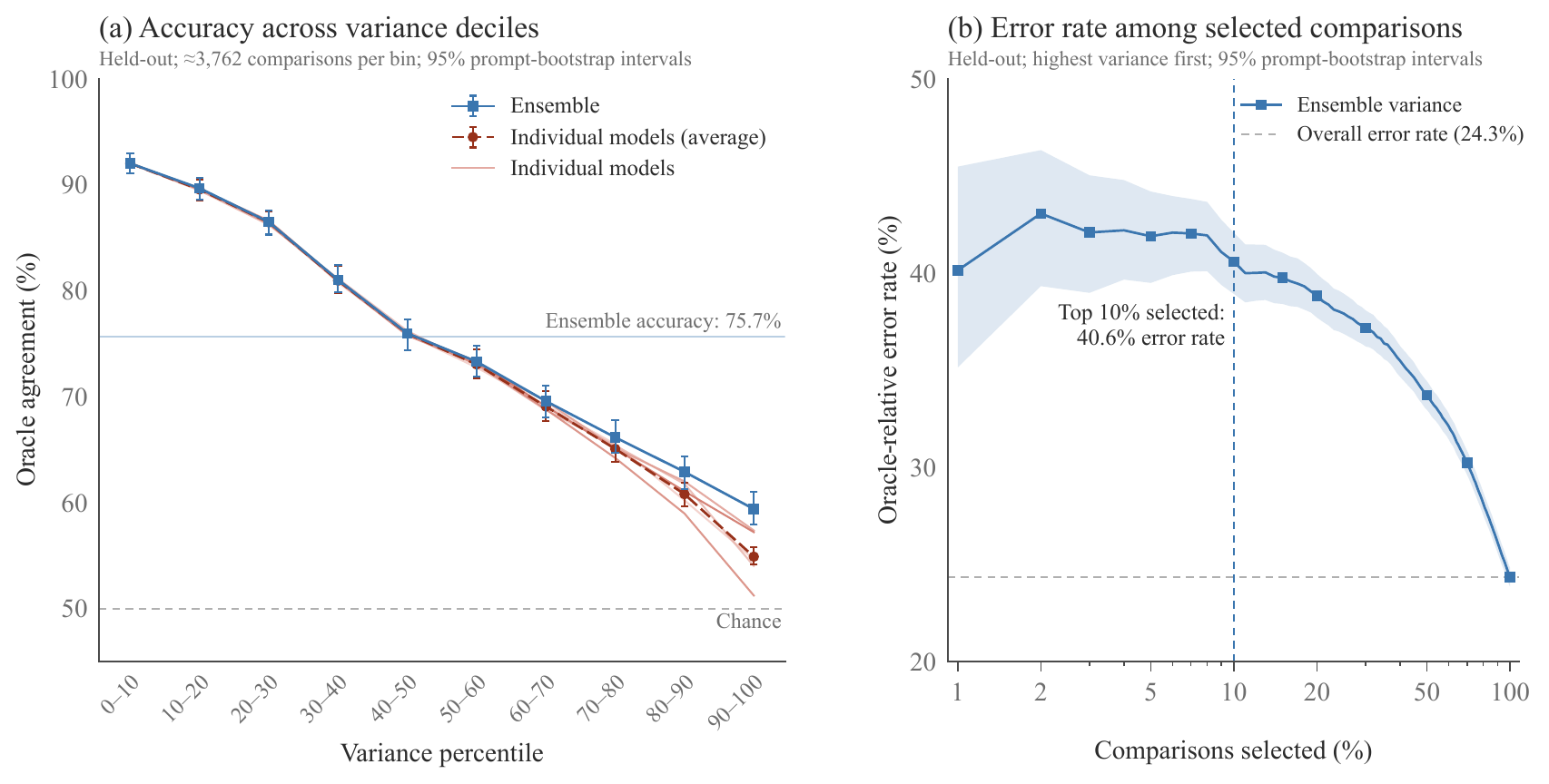}
    \caption{Ensemble variance and prediction difficulty on held out comparisons. \textbf{(a)}~Oracle agreement across variance deciles. \textbf{(b)}~Ensemble error rate among comparisons selected in decreasing order of variance. Error bars and shading show $95\%$ prompt-bootstrap intervals.}
    \label{fig:variance-proxy}
\end{figure}

Ensemble accuracy decreases monotonically across variance deciles, from $92.0\%$ in the lowest to $59.4\%$ in the highest (Figure~\ref{fig:variance-proxy}a). Since the Follower searches for high-variance triplets, error concentration in the upper tail is especially relevant: the top $10\%$ have an error rate of $40.6\%$, compared with $24.3\%$ overall (Figure~\ref{fig:variance-proxy}b). Across comparisons, variance achieves an area under the receiver operating characteristic curve (AUROC) of $0.670$ when ranking errors above correct predictions \citep{DBLP:journals/prl/Fawcett06}, indicating moderate discrimination. Other uncertainty metrics are provided in Appendix~\ref{sec:utility-as-design}.

\subsection{Generating Adversarial Responses} \label{sec:5.3}

We next ask whether the Leader--Follower game generates adversarial preference triplets on held-out prompts. Following Section~\ref{sec:adv-triplets-def}, we compare the trained and reference policies by measuring the frequency of $\alpha$-online adversarial candidates (Definition~\ref{def:alpha-online-adv}) and, using oracle labels, $\varepsilon$-adversarial triplets (Definition~\ref{def:epsilon-adv}).

All policies are initialized from the same base model, using the models and chat templates of Appendix~\ref{app:game-implementation}; the frozen initializations serve as $\piref$ and $\omegaref$. The utility is the ensemble variance of Section~\ref{sec:3-3.1}, rescaled to $[0,1]$ and computed using the frozen five-member ensemble of Section~\ref{app:preference-models}. We train with \textbf{\textsc{deep-smida}} (Algorithm~\ref{alg:deep-stackelberg}) on a set of prompts unseen during ensemble training, obtaining policies $(\widehat{\pi},\widehat{\omega})$. Appendix~\ref{app:game-implementation} provides hyperparameters and implementation details. We evaluate the trained pair $(\widehat{\pi},\widehat{\omega})$ and the reference pair $(\piref,\omegaref)$ on the same $2{,}500$ prompts, unseen during policy and ensemble training, with identical decoding settings. Each policy pair generates two response pairs per prompt, for a total of $5{,}000$ triplets. The empirical variance distribution of the reference triplets estimates $F_{\mathrm{Var}}$ and sets common candidate thresholds for both policy pairs. 

\subsubsection{Disagreement on Holdout Prompts}

We first assess whether the trained policies generate comparisons with greater ensemble disagreement on holdout prompts. Across the $5{,}000$ triplets generated by each policy pair, training increases mean ensemble variance from $0.0051$ to $0.0139$ ($95\%$ prompt-bootstrap interval: $[0.0133, 0.0145]$). The median rises from $0.00238$ to $0.00569$, and the trained pair produces higher mean variance on $68.5\%$ of evaluation prompts. The increase thus extends to the centre of the distribution and across prompts. Figure~\ref{fig:candidate-rate-and-margin}(a) shows how this shift translates into the yield of $\alpha$-online adversarial candidates, using common variance thresholds estimated from the reference generations. At $\alpha=0.1$, the trained pair gives $33.7\%$ candidates ($95\%$ interval: $[31.7\%,35.5\%]$), compared with $10.0\%$ for the reference pair. The relative gain increases as $\alpha$ decreases. At $\alpha=0.05$ and $\alpha=0.01$, trained gives are $26.0\%$ and $11.2\%$, respectively, or $5.2$ and $11.2$ times the reference rates. Training thus increases the yield of comparisons with high disagreement over the reference.

\begin{figure}[t]
    \centering
    \includegraphics[width=\linewidth]{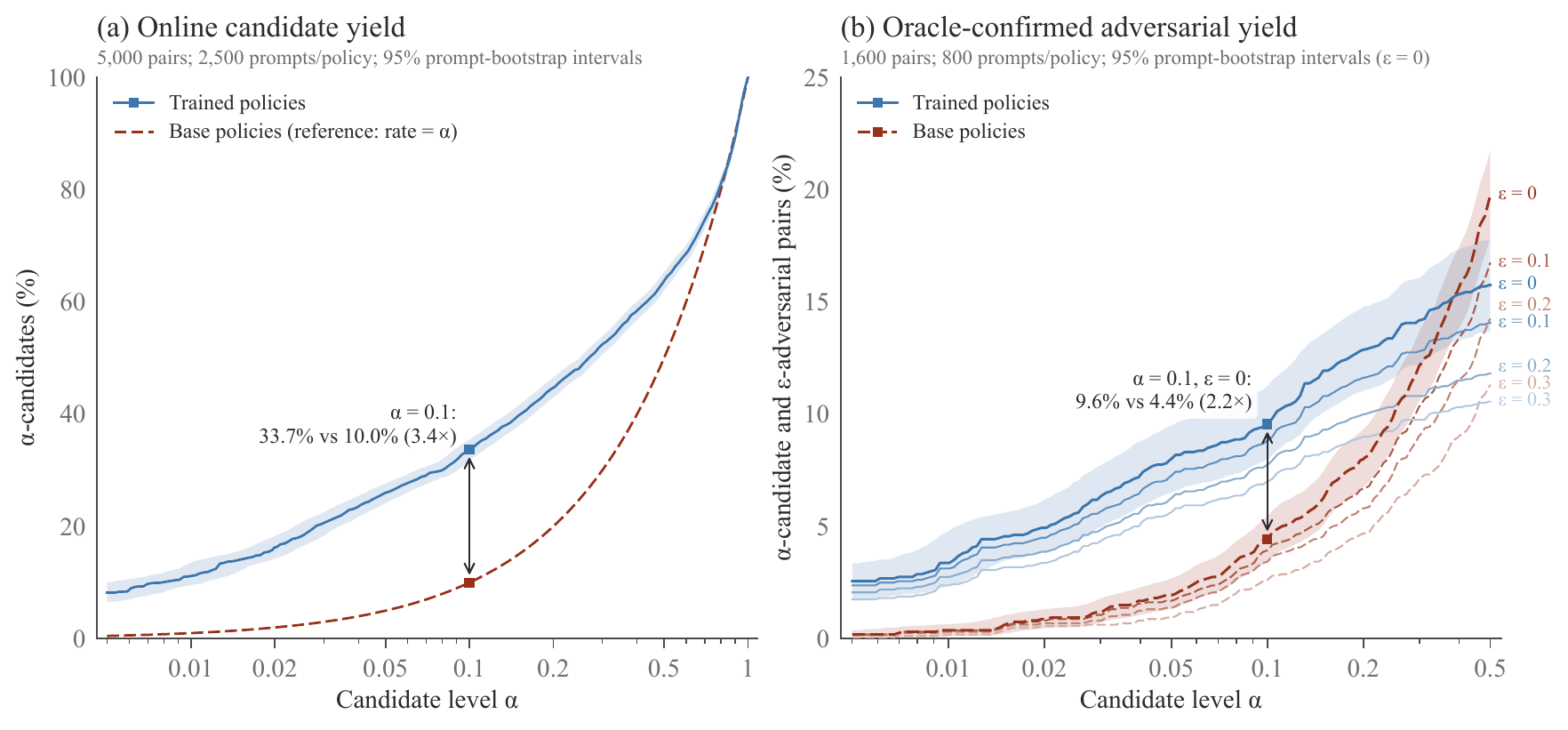}
    \caption{Online candidate yield and oracle-confirmed adversarial yield on holdout prompts. \textbf{(a)} Fraction of $5{,}000$ triplets per policy pair meeting the common reference-quantile threshold at level $\alpha$.
    \textbf{(b)} Fraction of all $1{,}600$ oracle-labelled triplets per policy pair that are both $\alpha$-online candidates and $\varepsilon$-adversarial. Blue solid curves denote trained policies; red dashed curves denote reference policies. Lighter curves indicate larger oracle margins. Shading shows $95\%$ prompt-bootstrap intervals}
    \label{fig:candidate-rate-and-margin}
\end{figure}

\subsubsection{Adversarial Triplets}\label{subsec:main-text-adver-exp}

Greater disagreement is useful for our purpose when it identifies comparisons on which the ensemble's mean preference is incorrect. We test this empirically on $1{,}600$ triplets from $800$ holdout prompts per policy pair, using the same synthetic oracle as in Section~\ref{subsec:meth}. For each triplet, we estimate $\mathcal{P}^*(y \succ y' \mid x)$ by averaging the valid binary verdicts from $11$ oracle queries, randomizing response order in each query. We apply the conditions of Definition~\ref{def:epsilon-adv}. Figure~\ref{fig:candidate-rate-and-margin}(b) reports the fraction of all labelled triplets that are both $\alpha$-online candidates and $\varepsilon$-adversarial. This joint yield measures how many oracle-confirmed errors pass the online criterion per generated triplet.

At $\alpha=0.1$ and $\varepsilon=0$, the trained pair gives $9.6\%$ ($95\%$ prompt-bootstrap interval: $[7.97, 11.25]$), compared with $4.4\%$ for the reference pair. At $\alpha=0.05$, the respective yields are $8.0\%$ ($95\%$ interval: $[6.51, 9.50]$) and $1.9\%$. The advantage persists for reversals with substantial oracle margins: at $\alpha=0.1$ and $\varepsilon=0.3$, the trained pair yields $7.0\%$ against $2.6\%$. These results support the game as a way to generate adversarial comparisons accessible to online selection.

The online criterion also concentrates errors within the trained policies' generations. At $\alpha=0.1$, $29.7\%$ of their candidates reverse the preferences, compared with $19.4\%$ of their generations overall, capturing $49.2\%$ of all their reversals. The reference pair nevertheless has a higher overall reversal rate ($34.3\%$) and the scope of this simple example is thus specific. We obtain a higher per-triplet yield of errors exceeding a disagreement threshold, with errors defined relative to the oracle.

\subsection{Further Experiments}\label{sec:main-text-further}

We briefly report two further experiments, with details in Appendix~\ref{app:further-exp}. First, we examine an indirect benefit of our method, the Leader's alignment with preferences. We conjecture that minimizing disagreement against an moving Follower may help preserve response quality. In fact, in a round-robin evaluation, our Leader achieves an oracle preference rate of $54.89\%$ against NLHF \citep{DBLP:conf/icml/MunosVCARGTGMFM24} and $50.70\%$ against the SLHF Leader \citep{pasztor2026stackelberg}. A Leader trained only to resist disagreement thus remains competitive with policies trained to be preferred (Appendix~\ref{app:roundrobin}).

To isolate the role of the Leader, we train the Follower only, while keeping the Leader fixed in \eqref{eq:first-adv-game}. We conjecture that a Follower facing a fixed Leader can exploit idiosyncratic patterns of that Leader's response distribution, rather than learn to construct useful responses. We therefore compare the two Followers on IFEval \citep{zhou2023instruction}, an instruction following benchmark. The Follower trained with a moving Leader achieves higher accuracy on all four aggregate metrics. Training against an adaptive Leader thus discourages strategies that induce disagreement at the expense of instruction following (Appendix~\ref{app:ifeval}).

\section{Conclusions, Broader Impact, and Limitations}

The goal of the present work is first and foremost to provide a framework for addressing a basic, almost naïve difficulty in alignment, that the evaluator is by necessity an imperfect proxy for human preferences and intents. Our Stackelberg formulation uses adversarial interaction to search for comparisons that expose these imperfections and provide candidates for additional supervision.  Due to compute constraints, our experiments are limited to small models and relatively simple tasks, with errors assessed against a synthetic oracle. Whether correcting the errors identified by this search makes the evaluator more robust remains to be established. A natural next step is therefore to connect adversarial search with repeated annotation and evaluator updates, and to test the resulting process on larger frontier models and on open-ended, long-horizon tasks whose outcomes are difficult to verify. More broadly, we hope that the ensemble-variance game illustrates a broader point, that the same optimizing behavior that exploits an evaluator can also help improve the same supervision on which alignment depends.

\subsection*{AI use statement}

In this work, we used generative AI tools to provide feedback on research methodology and experimental design, implement methods, cleaning and formatting the datasets, and check mathematical arguments. We also used a language model as a synthetic preference oracle to label response comparisons, as described in the experimental methodology. We have not used AI tools to propose or refine hypotheses, to interpret experimental results, or to set the direction of the research. Additionally, we used generative AI tools to identify, summarize, and analyze relevant literature, write and refine software code, create or modify scientific figures, and polish and contribute to the manuscript’s writing and editing. The models used include GPT-5.6 Sol, GPT-6 Astra (OpenAI) and Fable 5, Fable 5.1, Opus 5, Opus 5.5 (Anthropic). We, the authors, take full responsibility for the final contents of this work, including text, claims or artifacts produced with the aid of generative AI tools. 

\subsection*{Ethics statement}

We do not foresee significant ethical concerns arising from the experiments reported in this work. Our experiments involve no human annotators, all preference labels come from a synthetic oracle, applied to prompts from publicly available datasets. The errors we report are therefore defined relative to this oracle, which may or may not have biases of its own, and should not be read as departures from human judgment. The framework is intended to identify weaknesses in preference evaluators and to direct additional supervision toward them, with the broader aim of making alignment more reliable. 

\subsection*{Reproducibility statement}

All data used in this work are publicly available or will be made public. All code will be released upon acceptance. We refer the reader to the implementation Section~\ref{subsec:meth} and Appendix~\ref{app:deep} and \ref{app:implementation} for details of the experimental setup, training procedures, and hyperparameters.

\subsubsection*{Acknowledgments}

Use unnumbered third level headings for the acknowledgments. All acknowledgments, including those to funding agencies, go at the end of the paper.

\bibliography{iclr2027_conference}

@techreport{openai2026huggingface,
  author      = {{OpenAI}},
  title       = {{OpenAI}--{Hugging Face} Incident Technical Report},
  institution = {OpenAI},
  year        = {2026},
  month       = aug,
  url         = {https://cdn.openai.com/pdf/67869394-cb91-4c12-888c-5cbd85c7814c/OpenAI-Hugging-Face%20Incident-Technical-Report.pdf}
}

@techreport{metr2026huggingface,
  author      = {{METR}},
  title       = {Brief Independent Investigation of Agents' Behavior,
                 Reasoning and Collaboration in the {OpenAI} /
                 {Hugging Face} Hacking Incident},
  institution = {METR},
  year        = {2026},
  month       = aug,
  url         = {https://metr.org/hugging-face-incident-report-aug-2026.pdf}
}

@misc{deepseekai2026deepseekv4,
  title  = {{DeepSeek-V4}: Towards Highly Efficient Million-Token Context Intelligence},
  author = {{DeepSeek-AI}},
  year   = {2026},
  url    = {https://huggingface.co/deepseek-ai/DeepSeek-V4-Flash}
}

@article{qwen2024qwen25,
  title         = {{Qwen2.5 Technical Report}},
  author        = {{Qwen Team}},
  year          = {2024},
  journal = {arXiv preprint arXiv:2412.15115},
}

@book{aliprantis2006infinite,
  title={Infinite dimensional analysis: a hitchhiker’s guide},
  author={Aliprantis, Charalambos D. and Border, Kim C.},
  year={2006},
  publisher={Springer}
  
}

@article{amodei2016concrete,
  title   = {Concrete Problems in {AI} Safety},
  author  = {Amodei, Dario and Olah, Chris and Steinhardt, Jacob
             and Christiano, Paul and Schulman, John and Man{\'e}, Dan},
  journal = {arXiv preprint arXiv:1606.06565},
  year    = {2016},
}

@book{bostrom2014superintelligence,
  author    = {Bostrom, Nick},
  title     = {Superintelligence: Paths, Dangers, Strategies},
  year      = {2014},
  publisher = {Oxford University Press},
  address   = {Oxford},
  isbn      = {9780199678112}
}

@article{bubeck2015convex,
  author  = {Bubeck, S{\'e}bastien},
  title   = {Convex Optimization: Algorithms and Complexity},
  journal = {Foundations and Trends in Machine Learning},
  volume  = {8},
  number  = {3--4},
  pages   = {231--357},
  year    = {2015},
}

@inproceedings{
bukharin2025adversarial,
title={Adversarial Training of Reward Models},
author={Alexander Bukharin and Haifeng Qian and Shengyang Sun and Adithya Renduchintala and Soumye Singhal and Zhilin Wang and Oleksii Kuchaiev and Olivier Delalleau and Tuo Zhao},
booktitle={Second Conference on Language Modeling},
year={2025},
}

@inproceedings{DBLP:conf/iclr/ChakrabortyBKWM24,
  author       = {Souradip Chakraborty and
                  Amrit Singh Bedi and
                  Alec Koppel and
                  Huazheng Wang and
                  Dinesh Manocha and
                  Mengdi Wang and
                  Furong Huang},
  title        = {{PARL:} {A} Unified Framework for Policy Alignment in Reinforcement
                  Learning from Human Feedback},
  booktitle    = {The Twelfth International Conference on Learning Representations,
                  {ICLR}},
  year         = {2024},
  bibsource    = {dblp computer science bibliography, https://dblp.org}
}

@inproceedings{DBLP:conf/nips/ChristianoLBMLA17,
  author       = {Paul F. Christiano and
                  Jan Leike and
                  Tom B. Brown and
                  Miljan Martic and
                  Shane Legg and
                  Dario Amodei},
  title        = {Deep Reinforcement Learning from Human Preferences},
  booktitle    = {Advances in Neural Information Processing Systems 30: Annual Conference
                  on Neural Information Processing Systems},
  year         = {2017},
}

@inproceedings{DBLP:conf/nips/XuZJJ25,
  author       = {Xu Chu and Zhixin Zhang and Tianyu Jia and Yujie Jin},
  title        = {Stackelberg Self-Annotation: {A} Robust Approach to Data-Efficient
                  {LLM} Alignment},
  booktitle    = {Advances in Neural Information Processing Systems 38: Annual Conference
                  on Neural Information Processing Systems},
  year         = {2025},
  bibsource    = {dblp computer science bibliography, https://dblp.org}
}

@inproceedings{DBLP:conf/iclr/CosteAK024,
  author       = {Thomas Coste and
                  Usman Anwar and
                  Robert Kirk and
                  David Krueger},
  title        = {Reward Model Ensembles Help Mitigate Overoptimization},
  booktitle    = {The Twelfth International Conference on Learning Representations,
                  {ICLR}},
  year         = {2024},
  bibsource    = {dblp computer science bibliography, https://dblp.org}
}

@inproceedings{DBLP:conf/icml/CuiY0YH0NXXL0024,
  author       = {Ganqu Cui and
                  Lifan Yuan and
                  Ning Ding and
                  Guanming Yao and
                  Bingxiang He and
                  Wei Zhu and
                  Yuan Ni and
                  Guotong Xie and
                  Ruobing Xie and
                  Yankai Lin and
                  Zhiyuan Liu and
                  Maosong Sun},
  title        = {{ULTRAFEEDBACK:} Boosting Language Models with Scaled {AI} Feedback},
  booktitle    = {Forty-first International Conference on Machine Learning, {ICML}},
  series       = {Proceedings of Machine Learning Research},
  volume       = {235},
  year         = {2024},
  bibsource    = {dblp computer science bibliography, https://dblp.org}
}

@article{DBLP:journals/prl/Fawcett06,
  author       = {Tom Fawcett},
  title        = {An introduction to {ROC} analysis},
  journal      = {Pattern Recognit. Lett.},
  volume       = {27},
  number       = {8},
  pages        = {861--874},
  year         = {2006},
  bibsource    = {dblp computer science bibliography, https://dblp.org}
}

@inproceedings{DBLP:conf/icml/GaoSH23,
  author       = {Leo Gao and
                  John Schulman and
                  Jacob Hilton},
  title        = {Scaling Laws for Reward Model Overoptimization},
  booktitle    = {International Conference on Machine Learning, {ICML}},
  series       = {Proceedings of Machine Learning Research},
  volume       = {202},
  year         = {2023},
  bibsource    = {dblp computer science bibliography, https://dblp.org}
}

@article{gleave2022uncertainty,
  title   = {Uncertainty Estimation for Language Reward Models},
  author  = {Gleave, Adam and Irving, Geoffrey},
  journal = {arXiv preprint arXiv:2203.07472},
  year    = {2022}
}

@inproceedings{DBLP:conf/icml/GreenblattSSR24,
  author       = {Ryan Greenblatt and
                  Buck Shlegeris and
                  Kshitij Sachan and
                  Fabien Roger},
  title        = {{AI} Control: Improving Safety Despite Intentional Subversion},
  booktitle    = {Forty-first International Conference on Machine Learning, {ICML}},
  series       = {Proceedings of Machine Learning Research},
  volume       = {235},
  year         = {2024},
  bibsource    = {dblp computer science bibliography, https://dblp.org}
}

@article{gupta2025pessimism,
  title   = {Mitigating Preference Hacking in Policy Optimization with Pessimism},
  author  = {Gupta, Dhawal and Fisch, Adam and Dann, Christoph
             and Agarwal, Alekh},
  journal = {arXiv preprint arXiv:2503.06810},
  year    = {2025},
}

@inproceedings{eisenstein2024helping,
  title     = {Helping or Herding? Reward Model Ensembles Mitigate but do not Eliminate Reward Hacking},
  author    = {Eisenstein, Jacob and Nagpal, Chirag and Agarwal, Alekh and Beirami, Ahmad and D'Amour, Alex and Dvijotham, DJ and Fisch, Adam and Heller, Katherine and Pfohl, Stephen and Ramachandran, Deepak and Shaw, Peter and Berant, Jonathan},
  booktitle = {First Conference on Language Modeling},
  year      = {2024},
}

@inproceedings{DBLP:conf/nips/KalraB24,
  author       = {Dayal Singh Kalra and
                  Maissam Barkeshli},
  title        = {Why Warmup the Learning Rate? Underlying Mechanisms and Improvements},
  booktitle    = {Advances in Neural Information Processing Systems 37: Annual Conference
                  on Neural Information Processing Systems},
  year         = {2024},
  bibsource    = {dblp computer science bibliography, https://dblp.org}
}

@inproceedings{DBLP:conf/iclr/KoolHW19a,
  author       = {Wouter Kool and
                  Herke van Hoof and
                  Max Welling},
  title        = {Buy 4 {REINFORCE} Samples, Get a Baseline for Free!},
  booktitle    = {Deep Reinforcement Learning Meets Structured Prediction, {ICLR}},
  year         = {2019},
  bibsource    = {dblp computer science bibliography, https://dblp.org}
}

@inproceedings{DBLP:conf/nips/Lakshminarayanan17,
  author       = {Balaji Lakshminarayanan and
                  Alexander Pritzel and
                  Charles Blundell},
  title        = {Simple and Scalable Predictive Uncertainty Estimation using Deep Ensembles},
  booktitle    = {Advances in Neural Information Processing Systems 30: Annual Conference
                  on Neural Information Processing Systems},
  year         = {2017},
  bibsource    = {dblp computer science bibliography, https://dblp.org}
}

@inproceedings{DBLP:conf/icml/LeeSA21,
  author       = {Kimin Lee and
                  Laura M. Smith and
                  Pieter Abbeel},
  title        = {{PEBBLE:} Feedback-Efficient Interactive Reinforcement Learning via
                  Relabeling Experience and Unsupervised Pre-training},
  booktitle    = {Proceedings of the 38th International Conference on Machine Learning,
                  {ICML}},
  series       = {Proceedings of Machine Learning Research},
  volume       = {139},
  year         = {2021},
  bibsource    = {dblp computer science bibliography, https://dblp.org}
}

@inproceedings{DBLP:conf/iclr/LiangSLA22,
  author       = {Xinran Liang and
                  Katherine Shu and
                  Kimin Lee and
                  Pieter Abbeel},
  title        = {Reward Uncertainty for Exploration in Preference-based Reinforcement
                  Learning},
  booktitle    = {The Tenth International Conference on Learning Representations, {ICLR}
                  },
  year         = {2022},
  bibsource    = {dblp computer science bibliography, https://dblp.org}
}

@article{DBLP:journals/jmlr/Lin0J25,
  author       = {Tianyi Lin and
                  Chi Jin and
                  Michael I. Jordan},
  title        = {Two-Timescale Gradient Descent Ascent Algorithms for Nonconvex Minimax
                  Optimization},
  journal      = {J. Mach. Learn. Res.},
  volume       = {26},
  pages        = {11:1--11:45},
  year         = {2025},
  bibsource    = {dblp computer science bibliography, https://dblp.org}
}

@inproceedings{DBLP:conf/iclr/LoshchilovH19,
  author       = {Ilya Loshchilov and
                  Frank Hutter},
  title        = {Decoupled Weight Decay Regularization},
  booktitle    = {7th International Conference on Learning Representations, {ICLR}},
  year         = {2019},
  bibsource    = {dblp computer science bibliography, https://dblp.org}
}

@inproceedings{
makar-limanov2024starlhf,
title={{STA}-{RLHF}: Stackelberg Aligned Reinforcement Learning with Human Feedback},
author={Jacob Makar-Limanov and Arjun Prakash and Denizalp Goktas and Nora Ayanian and Amy Greenwald},
booktitle={Coordination and Cooperation for Multi-Agent Reinforcement Learning Methods Workshop},
year={2024},
}

@article{manheim2018categorizing,
  title   = {Categorizing Variants of {Goodhart's Law}},
  author  = {Manheim, David and Garrabrant, Scott},
  journal = {arXiv preprint arXiv:1803.04585},
  year    = {2018},
}

@article{may1954intransitivity,
  title={Intransitivity, utility, and the aggregation of preference patterns},
  author={May, Kenneth O},
  journal={Econometrica: Journal of the Econometric Society},
  pages={1--13},
   volume  = {22},
  year={1954},
  publisher={JSTOR}
}

@inproceedings{melikidze2026activeultrafeedback,
  title     = {{ActiveUltraFeedback}: Efficient Preference Data Generation using Active Learning},
  author    = {Melikidze, Davit and Schneider, Marian and Lam, Jessica and
               Wertich, Martin and Hakimi, Ido and P{\'a}sztor, Barna and Krause, Andreas},
  booktitle = {Proceedings of the 43rd International Conference on Machine Learning},
  series    = {Proceedings of Machine Learning Research},
  volume    = {306},
  year      = {2026},
  address   = {Seoul, South Korea},
}

@inproceedings{DBLP:conf/icml/MuldrewHZB24,
  author       = {William Muldrew and
                  Peter Hayes and
                  Mingtian Zhang and
                  David Barber},
  title        = {Active Preference Learning for Large Language Models},
  booktitle    = {Forty-first International Conference on Machine Learning, {ICML}},
  series       = {Proceedings of Machine Learning Research},
  volume       = {235},
  year         = {2024},
  bibsource    = {dblp computer science bibliography, https://dblp.org}
}

@inproceedings{DBLP:conf/icml/MunosVCARGTGMFM24,
  author       = {R{\'{e}}mi Munos and
                  Michal Valko and
                  Daniele Calandriello and
                  Mohammad Gheshlaghi Azar and
                  Mark Rowland and
                  Daniel Guo and
                  Yunhao Tang and
                  Matthieu Geist and
                  Thomas Mesnard and
                  C{\^{o}}me Fiegel and
                  Andrea Michi and
                  Marco Selvi and
                  Sertan Girgin and
                  Nikola Momchev and
                  Olivier Bachem and
                  Daniel J. Mankowitz and
                  Doina Precup and
                  Bilal Piot},
  title        = {Nash Learning from Human Feedback},
  booktitle    = {Forty-first International Conference on Machine Learning, {ICML}},
  series       = {Proceedings of Machine Learning Research},
  volume       = {235},
  year         = {2024},
  bibsource    = {dblp computer science bibliography, https://dblp.org}
}

@inproceedings{DBLP:conf/nips/Ouyang0JAWMZASR22,
  author       = {Long Ouyang and
                  Jeffrey Wu and
                  Xu Jiang and
                  Diogo Almeida and
                  Carroll L. Wainwright and
                  Pamela Mishkin and
                  Chong Zhang and
                  Sandhini Agarwal and
                  Katarina Slama and
                  Alex Ray and
                  John Schulman and
                  Jacob Hilton and
                  Fraser Kelton and
                  Luke Miller and
                  Maddie Simens and
                  Amanda Askell and
                  Peter Welinder and
                  Paul F. Christiano and
                  Jan Leike and
                  Ryan Lowe},
  title        = {Training language models to follow instructions with human feedback},
  booktitle    = {Advances in Neural Information Processing Systems 35: Annual Conference
                  on Neural Information Processing Systems},
  year         = {2022},
  bibsource    = {dblp computer science bibliography, https://dblp.org}
}

@inproceedings{DBLP:conf/iclr/PanBS22,
  author       = {Alexander Pan and
                  Kush Bhatia and
                  Jacob Steinhardt},
  title        = {The Effects of Reward Misspecification: Mapping and Mitigating Misaligned
                  Models},
  booktitle    = {The Tenth International Conference on Learning Representations, {ICLR}
            },
  year         = {2022},
  bibsource    = {dblp computer science bibliography, https://dblp.org}
}

@inproceedings{
pasztor2026stackelberg,
title={Stackelberg Learning from Human Feedback: Preference Optimization as a Sequential Game},
author={Barna P{\'a}sztor and Thomas Kleine Buening and Andreas Krause},
booktitle={The Fourteenth International Conference on Learning Representations},
year={2026},
}

@inproceedings{DBLP:conf/nips/RafailovSMMEF23,
  author       = {Rafael Rafailov and
                  Archit Sharma and
                  Eric Mitchell and
                  Christopher D. Manning and
                  Stefano Ermon and
                  Chelsea Finn},
  title        = {Direct Preference Optimization: Your Language Model is Secretly a
                  Reward Model},
  booktitle    = {Advances in Neural Information Processing Systems 36: Annual Conference
                  on Neural Information Processing Systems},
  year         = {2023},
  bibsource    = {dblp computer science bibliography, https://dblp.org}
}

@book{russell2019human,
  author    = {Russell, Stuart},
  title     = {Human Compatible: Artificial Intelligence and the Problem of Control},
  publisher = {Viking},
  year      = {2019},
  isbn      = {9780525558613}
}

@inproceedings{DBLP:conf/colt/SeungOS92,
  author       = {H. Sebastian Seung and
                  Manfred Opper and
                  Haim Sompolinsky},
  title        = {Query by Committee},
  booktitle    = {Proceedings of the Fifth Annual {ACM} Conference on Computational
                  Learning Theory, {COLT}},
  pages        = {287--294},
  publisher    = {{ACM}},
  year         = {1992},
  bibsource    = {dblp computer science bibliography, https://dblp.org}
}

@article{simaan1973stackelberg,
  author  = {Marwan Simaan and Jose B. Cruz},
  title   = {On the {S}tackelberg Strategy in Nonzero-Sum Games},
  journal = {Journal of Optimization Theory and Applications},
  volume  = {11},
  number  = {5},
  pages   = {533--555},
  year    = {1973}
}

@article{sion1958general,
  author  = {Sion, Maurice},
  title   = {On general minimax theorems},
  journal = {Pacific Journal of Mathematics},
  volume  = {8},
  number  = {1},
  pages   = {171--176},
  year    = {1958},
}

@inproceedings{DBLP:conf/nips/SkalseHKK22,
  author       = {Joar Skalse and
                  Nikolaus H. R. Howe and
                  Dmitrii Krasheninnikov and
                  David Krueger},
  title        = {Defining and Characterizing Reward Gaming},
  booktitle    = {Advances in Neural Information Processing Systems 35: Annual Conference
                  on Neural Information Processing Systems},
  year         = {2022},
  bibsource    = {dblp computer science bibliography, https://dblp.org}
}

@inproceedings{DBLP:conf/icml/SwamyDK0A24,
  author       = {Gokul Swamy and
                  Christoph Dann and
                  Rahul Kidambi and
                  Steven Wu and
                  Alekh Agarwal},
  title        = {A Minimaximalist Approach to Reinforcement Learning from Human Feedback},
  booktitle    = {Forty-first International Conference on Machine Learning, {ICML}},
  series       = {Proceedings of Machine Learning Research},
  volume       = {235},
  year         = {2024},
  bibsource    = {dblp computer science bibliography, https://dblp.org}
}

@article{tversky1969intransitivity,
  title={Intransitivity of preferences.},
  author={Tversky, Amos},
  journal={Psychological review},
  volume={76},
  number={1},
  pages={31--48},
  year={1969},
  publisher={American Psychological Association}
}

@book{vonstackelberg1934marktform,
  author    = {Heinrich von Stackelberg},
  title     = {Marktform und Gleichgewicht},
  publisher = {Julius Springer},
  address   = {Wien and Berlin},
  year      = {1934},
  note      = {English translation: \emph{Market Structure and Equilibrium}, trans. D.~Bazin, L.~Urch and R.~Hill, Springer, 2011}
}

@article{wiener1960some,
  title={Some moral and technical consequences of automation: As machines learn they may develop unforeseen strategies at rates that baffle their programmers.},
  author={Wiener, Norbert},
  journal={Science},
  volume={131},
  number={3410},
  pages={1355--1358},
  year={1960},
  publisher={American Association for the Advancement of Science}
}

@inproceedings{DBLP:conf/iclr/WuSYJYG25,
  author       = {Yue Wu and
                  Zhiqing Sun and
                  Huizhuo Yuan and
                  Kaixuan Ji and
                  Yiming Yang and
                  Quanquan Gu},
  title        = {Self-Play Preference Optimization for Language Model Alignment},
  booktitle    = {The Thirteenth International Conference on Learning Representations,
                  {ICLR}},
  year         = {2025},
  bibsource    = {dblp computer science bibliography, https://dblp.org}
}

@inproceedings{DBLP:conf/icml/XuRRK25,
  author       = {Yi Xu and
                  Laura Ruis and
                  Tim Rockt{\"{a}}schel and
                  Robert Kirk},
  title        = {Investigating Non-Transitivity in {LLM}-as-a-Judge},
  booktitle    = {Forty-second International Conference on Machine Learning, {ICML}
            },
  series       = {Proceedings of Machine Learning Research},
  volume       = {267},
  year         = {2025},
  bibsource    = {dblp computer science bibliography, https://dblp.org}
}

@inproceedings{DBLP:conf/icml/ZhangZWXG25,
  author       = {Yifan Zhang and
                  Ge Zhang and
                  Yue Wu and
                  Kangping Xu and
                  Quanquan Gu},
  title        = {Beyond {Bradley-Terry} Models: {A} General Preference Model for Language
                  Model Alignment},
  booktitle    = {Forty-second International Conference on Machine Learning, {ICML}
            },
  series       = {Proceedings of Machine Learning Research},
  volume       = {267},
  year         = {2025},
  bibsource    = {dblp computer science bibliography, https://dblp.org}
}

@inproceedings{Lewis1994ASA,
  title={A sequential algorithm for training text classifiers},
  author={David D. Lewis and William A. Gale},
  booktitle={Annual International ACM SIGIR Conference on Research and Development in Information Retrieval},
  year={1994},
}

@article{Lin1991DivergenceMB,
  title={Divergence measures based on the {S}hannon entropy},
  author={Jianhua Lin},
  journal={IEEE Trans. Inf. Theory},
  year={1991},
  volume={37},
  pages={145-151},
}

@article{zhou2023instruction,
  title={Instruction-following evaluation for large language models},
  author={Zhou, Jeffrey and Lu, Tianjian and Mishra, Swaroop and Brahma, Siddhartha and Basu, Sujoy and Luan, Yi and Zhou, Denny and Hou, Le},
  journal={arXiv preprint arXiv:2311.07911},
  year={2023}
}
\bibliographystyle{iclr2027_conference}

\appendix

\section{Extended Literature Review}\label{app:lit-rev}

\textbf{Alignment.} Aligning language models with human preferences requires eliciting those preferences and translating them into a learning signal. Reinforcement learning from human feedback (RLHF) learns scalar rewards from human comparisons and optimizes policies against them \citep{DBLP:conf/nips/ChristianoLBMLA17, DBLP:conf/nips/Ouyang0JAWMZASR22}. A scalar reward induces a transitive ordering, whereas pairwise whereas aggregated human judgments and LLM judges are routinely intransitive \citep{may1954intransitivity, tversky1969intransitivity, DBLP:conf/icml/XuRRK25}. General pairwise preference models accommodate such intransitivities \citep{DBLP:conf/icml/ZhangZWXG25}. Nash learning from human feedback (NLHF) \citep{DBLP:conf/icml/MunosVCARGTGMFM24} and its descendants \citep{DBLP:conf/icml/SwamyDK0A24,DBLP:conf/iclr/WuSYJYG25} seek to compute a Nash equilibrium of simultaneous play by optimizing regularized preference win rates under a fixed preference model.

\textbf{Reward Hacking.} A learned evaluator is only an imperfect proxy for human preferences and optimizing against it can increase its scores without producing outcomes that humans prefer \citep{amodei2016concrete,DBLP:conf/nips/SkalseHKK22}. Stronger optimization can exacerbate this mismatch across control tasks and language models, increasing proxy reward while reducing true (synthetic) reward \citep{DBLP:conf/iclr/PanBS22,DBLP:conf/icml/GaoSH23}; general preference models remain vulnerable \citep{gupta2025pessimism}.

\textbf{Active Learning.} Active preference learning selects comparisons for annotation using uncertainty scores, such as predictive entropy \citep{DBLP:conf/icml/MuldrewHZB24}, or Thompson sampling over fixed response pools \citep{melikidze2026activeultrafeedback}. RUNE \citep{DBLP:conf/iclr/LiangSLA22} instead adds ensemble standard deviation to the agent's reward, optimizing a policy toward disagreement. Adv-RM \citep{bukharin2025adversarial} similarly trains a generator to maximize disagreement between two reward models while keeping the target model's reward high. We optimize disagreement directly, training two opposing policies against an ensemble without this directional constraint. By generating comparisons as the policies evolve, our approach also searches for failure modes that fixed response pools may miss.

\textbf{Stackelberg Learning.} Prior Stackelberg alignment methods optimize rewards while anticipating policy adaptation \citep{DBLP:conf/iclr/ChakrabortyBKWM24}, adapt a preference model to a leading policy \citep{makar-limanov2024starlhf}, or optimize against worst-case preference-gap distributions \citep{DBLP:conf/nips/XuZJJ25}. Closest to our formulation, Stackelberg learning from human feedback (SLHF) \citep{pasztor2026stackelberg} trains a Follower to generate a preferred alternative conditional on the Leader's realized response. 

\section{Existence and Uniqueness of Stackelberg Equilibria}\label{app-th}

This Appendix contains the proof(s) of the existence and uniqueness Theorem~\ref{th:existence-uniqueness-main-text}. We do so in steps. First, we establish equilibrium existence for continuous convex-concave objectives. We then recall a variational identity for KL-regularized optimization and use it to derive the equilibrium policies in closed form and prove uniqueness when both regularization parameters are strictly positive. 

Begin by observing that as $\mathcal{X}$ and $\mathcal{Y}$ are finite, the strategy sets $\Pi$ and $\Omega$ are nonempty convex compact subsets of finite-dimensional Euclidean spaces. For the reader convenience, we repeat the definition of Stackelberg Equilibrium in a slightly more general form. We also note that the specific form of \eqref{eq:regularized-eq} is, in fact, not needed to prove existence. 

\begin{definition}
    Let $f:\Pi \times \Omega \to \mathbb{R}$ be a function. A pair $(\pi^*, \omega^*)$ is a $f$-Stackelberg Equilibrium for the zero-sum game with utility $f$ if
    \begin{equation*}
         \pi^* \in \argmin_{\pi \in \Pi} \, \max_{\omega \in \Omega} \, f(\pi, \omega), \quad \omega^* \in \argmax_{\omega \in \Omega}  f(\pi^*, \omega).
    \end{equation*}
\end{definition}

\begin{theorem}\label{th:existence}
 Let $f:\Pi \times \Omega \to \mathbb{R}$ be continuous convex-concave\footnote{Such a function $f$ is said to be convex-concave if for all $\omega \in \Omega$ the section $f(\cdot, \omega)$ is convex and similarly for all $\pi \in \Pi$ the section $f(\pi, \cdot)$ is concave.}. Then there exists $(\pi^*, \omega^*)$ such that
 \begin{equation*}
     f(\pi^*, \omega) \le f(\pi^*, \omega^*) \le f(\pi, \omega^*), \quad \forall \pi \in \Pi, \omega \in \Omega.
 \end{equation*}
 In particular, $(\pi^*, \omega^*)$ is a Stackelberg equilibrium and 
\begin{equation*}
    \max_\omega \min_\pi f(\pi, \omega)= \min_\pi \max_\omega f(\pi, \omega).
\end{equation*}
\end{theorem}

\begin{proof}
    Define the lower and upper value of the zero-sum game by
    \begin{equation*}
        \underline{v} \coloneq \max_\omega \min_\pi f(\pi, \omega), \quad \overline{v} \coloneq \min_\pi \max_\omega f(\pi, \omega).
    \end{equation*}
    The immediate inequality gives $\underline{v}   \le \overline{v}$ and by the VnM-Sion maxmin Theorem~\cite{sion1958general} we have equality $\underline{v} = \overline{v}$; denote the common value by $v$. By compactness of $\Pi$ and $\Omega$ and by continuity of the maps $\omega \mapsto \min_\pi f(\pi, \omega)$ and $\pi \mapsto \max_\omega f(\pi, \omega)$, the latter which follows by an application of Berge's Maximum Theorem (see Theorem~17.31 of \cite{aliprantis2006infinite}), we can choose $\pi^* \in \Pi$ and $\omega^* \in \Omega$ such that
     \begin{equation*}
        \min_\pi f(\pi, \omega^*) = v, \quad \max_\omega f(\pi^*, \omega) = v.
    \end{equation*}
    In particular, by the first condition we have that $f(\pi^*, \omega^*) \ge v$ and symmetrically from the second we have $f(\pi^*, \omega^*) \le v$, hence equality. 
    Combining the previous two equalities gives the saddle property
\begin{equation*}
     f(\pi^*, \omega) \le f(\pi^*, \omega^*) \le f(\pi, \omega^*), \quad \forall \pi \in \Pi, \omega \in \Omega.
 \end{equation*}
  By the left-hand inequality, $\omega^* \in \argmax_{\omega \in \Omega}  f(\pi^*, \omega)$. Finally, 
 \begin{equation*}
    \max_\omega f(\pi^*, \omega) =f(\pi^*, \omega^*) = v = \min_\pi \max_\omega f(\pi, \omega),
 \end{equation*}
 which gives $\pi^* \in \argmin_{\pi \in \Pi} \, \max_{\omega \in \Omega} \, f(\pi, \omega)$ as required.
\end{proof}

Note that the function $f:\Pi \times \Omega \to \mathbb{R}$ as defined in Equation~\eqref{eq:regularized-eq} is continuous and convex-concave, hence satisfies the assumptions of Theorem~\ref{th:existence}; this proves the first part of Theorem~\ref{th:existence-uniqueness-main-text}. We next exploit the explicit structure of \eqref{eq:regularized-eq} to prove uniqueness and to obtain an explicit expression for the equilibrium policies. The following Lemma is a well known folklore result, of which, for completeness sake, we report a simple proof. For similar arguments, see also \cite{DBLP:conf/nips/RafailovSMMEF23}.

\begin{lemma}\label{lem:variational-minimization}
    Let $A$ be a finite set, let $r \in \Delta(A)$ be of full support, let $\tau \in \mathbb{R}_{>0}$ and $h:A \to \mathbb{R}$. Then 
    \begin{equation}
        \min_{q \in \Delta(A)} \{ \langle q, h \rangle + \tau \KL (q \,\|\, r) \} = -\tau \log Z,
    \end{equation}
    where $\langle\cdot, \cdot \rangle$ denotes the expectation and $Z$ is the normalizing constant
    \begin{equation}
        Z \coloneq \sum_{a\in A}r(a)e^{-h(a)/\tau}.
    \end{equation}
    Moreover, the minimizer is unique and is given by
    \begin{equation}\label{eq:appA:minimizer}
        q^*(a)=\frac{r(a)e^{-h(a)/\tau}}{Z}.
    \end{equation}
\end{lemma}

\begin{proof}
Since $r$ has full support, $q^*$ as defined in \eqref{eq:appA:minimizer} also has full support. For any
\(q \in \Delta(A)\),
\begin{align*}
    \KL(q \,\|\, q^*)
    &= \sum_{a \in A} q(a)
       \log\left(\frac{q(a)}{q^*(a)}\right) \\
    &= \KL(q \,\|\, r)
       + \frac{1}{\tau}\langle q,h\rangle
       + \log Z.
\end{align*}
Rearranging gives
\begin{equation}\label{eq:appendix-variational}
    \langle q,h\rangle
    + \tau \KL(q \,\|\, r)
    =
    -\tau\log Z
    + \tau\KL(q \,\|\, q^*).
\end{equation}
Because the Kullback--Leibler divergence is nonnegative, the right-hand side is minimized at $q^*$, with minimum value $-\tau\log Z$. Moreover, $\KL(q \,\|\, q^*)=0$ if and only if $q=q^*$, which proves uniqueness.
\end{proof}

\begin{theorem}\label{th:uniqueness-explicit}
    Suppose that $\mathcal{X}$ and $\mathcal{Y}$ are finite and that $\rho$, $\piref$, and $\omegaref$ have full support. Let $f$ be defined as in \eqref{eq:regularized-eq}. Suppose that $\tau^L, \tau^F> 0$ and $\piref, \omegaref$ are of full support. Then there exists a unique $f$-Stackelberg Equilibrium $(\pi^*, \omega^*)$ given by
\begin{equation*}
\begin{aligned}
\omega^*(y' \mid x,y)
&=
\frac{
    \omegaref(y' \mid x,y)
    \exp\!\left(u(x,y,y')/\tau^F\right)
}{
    Z^F(x,y)
}, \\
\pi^*(y \mid x)
&=
\frac{
    \piref(y \mid x)
    \exp\!\left(
        -\frac{\tau^F}{\tau^L}\log Z^F(x,y)
    \right)
}{
    Z^L(x)
},
\end{aligned}
\end{equation*}
where the normalization functions $Z^F$ and $Z^L$ are defined as
\begin{align*}
        Z^F(x,y)
        &\coloneqq
        \sum_{z\in\mathcal{Y}}
        \omegaref(z\mid x,y)
        \exp\!\bigl(u(x,y,z)/\tau^F\bigr),\\
        Z^L(x)
        &\coloneqq
        \sum_{z\in\mathcal{Y}}
        \piref(z\mid x)
        \exp\!\bigl(-(\tau^F/\tau^L)\log Z^F(x,z)\bigr).
\end{align*}
\end{theorem}

\begin{proof}
Fix a Leader's policy $\pi \in \Pi$. The Follower's problem is equivalent to
\begin{equation}\label{eq:appA:proof-uniquenes1}
    \max_{\omega \in \Omega} \mathbb{E}_{x \sim \rho(\cdot), y \sim \pi(\cdot\mid x)}[\Phi_{x,y}(\omega(\cdot\mid x,y)],
\end{equation}
where we define
\begin{equation*}
    \Phi_{x,y}(\nu)\coloneq \mathbb{E}_{y'\sim \nu(\cdot)}[u(x,y,y')]-\tau^F\operatorname{KL}(\nu \,\|\, \omegaref(\cdot\mid x,y)).
\end{equation*}
    We solve the maximization problem \eqref{eq:appA:proof-uniquenes1} pointwise for each $(x,y) \in \mathcal{X} \times \mathcal{Y}$. By Lemma~\ref{lem:variational-minimization}, we have
    \begin{equation*}
        \max_{\nu\in\Delta(\mathcal Y)}
\Phi_{x,y}(\nu)
=
\tau^F\log Z^F(x,y),
    \end{equation*}
with unique maximizer
\begin{equation*}
    \omega^*(y'\mid x,y)
=
\frac{
\omegaref(y'\mid x,y)
\exp\!\left(u(x,y,y')/\tau^F\right)
}{
Z^F(x,y)
},
\end{equation*}
where the normalization function $Z^F$ is defined as
\begin{equation*}
    Z^F(x,y)
\coloneq
\sum_{y'\in\mathcal Y}
\omegaref(y'\mid x,y)
\exp\!\left(\frac{u(x,y,y')}{\tau^F}\right).
\end{equation*}
It follows that $\omega^*$ maximizes $f(\pi, \cdot)$ for every $\pi$. For a fixed $\pi$, this maximizer is unique at every point $(x,y)$ such that $\rho(x)\pi(y \mid x)>0$. Substituting the Follower’s optimal value into the Leader’s problem gives
\begin{equation}\label{eq:appA:proof-uniquenes2}
\begin{aligned}
V(\pi)
&\coloneq
\max_{\omega\in\Omega}f(\pi,\omega)\\
&\coloneq
\sum_{x\in\mathcal X}\rho(x)
\left[
\sum_{y\in\mathcal Y}
\pi(y \mid x)\tau^F\log Z^F(x,y)
+
\tau^L
\mathrm{KL}\!\left(
\pi(\cdot \mid x)
\,\|\,
\piref(\cdot \mid x)
\right)
\right].
\end{aligned}
\end{equation}
The minimization of $V$ separates across contexts. For each $x \in \mathcal{X}$, apply Lemma~\ref{lem:variational-minimization} again, now with
\begin{equation*}
    Z^L(x) \coloneq
\sum_{y\in\mathcal Y}
\piref(y \mid x)
\exp\!\left(
-\frac{\tau^F}{\tau^L}\log Z^F(x,y)
\right).
\end{equation*}
Lemma~\ref{lem:variational-minimization} shows that the minimum of the $x$-indexed term in \eqref{eq:appA:proof-uniquenes2} is $-\tau^L\log Z^L(x)$, and that its unique minimizer is
\begin{equation*}
    \pi^*(y \mid x)
=
\frac{
\piref(y \mid x)
\exp\!\left(
-\frac{\tau^F}{\tau^L}\log Z^F(x,y)
\right)
}{
Z^L(x)
}.
\end{equation*}
Since $\rho$ has full support, every context-indexed term in \eqref{eq:appA:proof-uniquenes2} has strictly positive weight. Hence $\pi^*$ is the unique minimizer of $V$, and, in particular, 
\begin{equation*}
\pi^* \in \argmin_{\pi \in \Pi} \, \max_{\omega \in \Omega} \, f(\pi, \omega).
\end{equation*}
Moreover, since $\piref$ has full support, so does $\pi^*$. Thus, for every $(x,y)$, we have $\rho(x)\pi^*(y\mid x)>0$. By the preceding argument, this implies that $\omega^*$ is the unique maximizer of $f(\pi^*,\cdot)$, completing the proof.
\end{proof}

\section{Convergence of Stackelberg Mirror Descent--Ascent}\label{app:tabular-algo}

This Appendix proves a stronger version of Theorem~\ref{thm:smida-main}, on the last-iterate convergence of \textbf{\textsc{smida}}. We establish explicit bounds in both directions of the KL divergence for each player, uniformly over contexts and, for the Follower, conditioning actions.

\subsection{Uniform Last Iterate Convergence}

For a function $v:A \to \mathbb{R}$, write $\operatorname{osc}(v):=\max_a v(a)-\min_a v(a)$ for the oscillation of $v$. Define the two payoff oscillations
\begin{align*}
  D_F
  &\coloneq\max_{x,y}\operatorname{osc}\left(
    u(x,y, \cdot)
  \right)= \max_{x,y}\left[
    \max_{y'}u(x,y,y')-\min_{y'}u(x,y,y')
  \right],\\
  D
  &\coloneq
  \max_{x}\operatorname{osc}\left(
    u(x, \cdot, \cdot)
  \right)= \max_x\left[
    \max_{y,y'}u(x,y,y')-\min_{y,y'}u(x,y,y')
  \right],
\end{align*}
The first is the variation relevant to the Follower's tracking error; the second also allows the Leader action to vary, and satisfies $D_F\leq D$.

\begin{theorem}\label{thm:smida-last-iterate}
  Suppose that $\mathcal X$ and $\mathcal Y$ are finite, that $\rho$, $\piref$, and $\omegaref$ have full support, and that $\tau^L,\tau^F>0$.  Run Algorithm~\ref{alg:smida} under condition~\eqref{eq:step-sizes}, and let $(\pi^*,\omega^*)$ be the unique regularized Stackelberg equilibrium.  Set
  \begin{equation*}
    q_L:=1-\eta^L\tau^L,
    \qquad
    q_F:=1-\eta^F\tau^F,
    \qquad
    \kappa:=\frac{\eta^F\tau^F}{\eta^L\tau^L}>1,
    \label{eq:smida-rates}
  \end{equation*}
  and define
  \begin{equation*}
    r(x,y):=\tau^F\log\!\sum_{y'}
    \omegaref(y'\mid x,y)
    \exp\!\left(\frac{u(x,y,y')}{\tau^F}\right),
    \qquad
    D_L:=\max_x\operatorname{osc}r(x,\cdot).
    \label{eq:smida-exact-cost}
  \end{equation*}
  Then all iterates have full support and, for every $t\geq0$,
  \begin{align}
    &\max_{x,y}\max\!\left\{
      \KL\!\left(\omega^*(\cdot\mid x,y)\|\omega_t(\cdot\mid x,y)\right),
      \KL\!\left(\omega_t(\cdot\mid x,y)\|\omega^*(\cdot\mid x,y)\right)
    \right\}
    \leq q_F^t\frac{D_F}{\tau^F},
    \label{eq:smida-follower-bound}\\
    &\max_x\max\!\left\{
      \KL\!\left(\pi^*(\cdot\mid x)\|\pi_t(\cdot\mid x)\right),
      \KL\!\left(\pi_t(\cdot\mid x)\|\pi^*(\cdot\mid x)\right)
    \right\}
    \nonumber\\[-2pt]
    &\hspace{35mm}\leq
      q_L^t\frac{D_L}{\tau^L}
      +\eta^L D_F
      \frac{q_L^t-q_F^t}{q_L-q_F}.
    \label{eq:smida-leader-exact-bound}
  \end{align}
  Moreover, $D_L\leq D$, and hence the simpler bound
  \begin{equation}
    \max_x\max\!\left\{
      \KL(\pi^*(\cdot\mid x)\|\pi_t(\cdot\mid x)),
      \KL(\pi_t(\cdot\mid x)\|\pi^*(\cdot\mid x))
    \right\}
    \leq
    q_L^t\frac{D}{\tau^L}\frac{\kappa}{\kappa-1}
    \label{eq:smida-leader-simple-bound}
  \end{equation}
  also holds.  Thus the Follower converges at rate $q_F^t$ and the Leader at rate $q_L^t$, simultaneously in both directions of KL and uniformly over all contexts. 
\end{theorem}

\begin{proof}
    Observe that, by their definition in Algorithm~\ref{alg:smida}, all iterates have full support. We make use of the following inequality:
    \begin{equation}\label{eq:app:th-sda1}
    \max\{\mathrm{KL}(p\|q),\mathrm{KL}(q\|p)\}
    \leq
    \operatorname{osc}\!\left(\log p-\log q\right)
  \end{equation}
  for full-support distributions $p,q$. Indeed, if $v:=\log p-\log q$, normalization implies $\min v\leq0\leq\max v$. Therefore $\mathrm{KL}(p\|q)=\mathbb E_p[v]\leq\max v\leq \operatorname{osc}(v)$, and the reverse direction follows identically.

    For the Follower, let $v_t^{x,y}:=\log\omega_t(\cdot\mid x,y)-
  \log\omega^\ast(\cdot\mid x,y)$, where $\omega^*$ is the unique equilibrium policy, as per Theorem~\ref{th:uniqueness-explicit}.  Taking logarithms in the Follower update gives
  \begin{equation*}
    v_{t+1}^{x,y}=q_Fv_t^{x,y}+c_t^{x,y}\mathbf 1
  \end{equation*}
  for a scalar normalizing constant $c_t^{x,y}$ and $\mathbf 1$ the vector of ones.  Since oscillation is
  invariant under addition of constants,
  \begin{equation}\label{eq:smida-follower-oscillation-rate}
    \operatorname{osc}(v_t^{x,y})
    =q_F^t\operatorname{osc}(v_0^{x,y})
    =q_F^t\frac{\operatorname{osc}u(x,y,\cdot)}{\tau^F}
    \leq q_F^t\frac{D_F}{\tau^F}.
  \end{equation}
  Combining this with~\eqref{eq:app:th-sda1} proves
  \eqref{eq:smida-follower-bound}.

From the explicit form of $\omega^*$, we have
$\log(\omega^\ast/\omega^{\mathrm{ref}})=u/\tau^F-\log Z^F$. Hence, 
\[
\tau^F\mathrm{KL}(\omega_t\|\omega^*)
=\tau^F\mathrm{KL}(\omega_t\|\omegaref)
-\mathbb E_{\omega_t}[u]+\tau^F\log Z^F.
\]
Using $r=\tau^F\log Z^F$ and $\widehat r_t=\mathbb E_{\omega_t}[u]-\tau^F\mathrm{KL}(\omega_t\|\omega^{\mathrm{ref}})$ gives
$r-\widehat r_t=\tau^F\mathrm{KL}(\omega_t\|\omega^\ast)$. Equations~\eqref{eq:app:th-sda1} and \eqref{eq:smida-follower-oscillation-rate} imply, uniformly in $(x,y)$,
  \begin{equation*}\label{eq:smida-cost-error-rate}
    0\leq r(x,y)- \widehat{r}_t(x,y) \leq D_Fq_F^t.
  \end{equation*}
  For the Leader, put $m_t^x:=\log\pi_t(\cdot\mid x)-\log\pi^\ast(\cdot\mid x)$. Taking logarithms in the Leader update gives
  \begin{equation*}
    m_{t+1}^x
    =q_Lm_t^x+\eta^L(r(x,\cdot)- \widehat{r}_t(x,\cdot))+d_t^x\mathbf1
    \label{eq:smida-leader-log-recursion}
  \end{equation*}
  for another scalar normalizer $d_t^x$. Because $r-\widehat{r}_t \geq0$, we have $\operatorname{osc}(r-\widehat{r}_t)(x,\cdot)\leq\max_y(r-\widehat{r}_t)(x,y)$. Hence
  \begin{equation*}
    \max_x\operatorname{osc}(m_{t+1}^x)
    \leq
    q_L\max_x\operatorname{osc}(m_t^x)
    +\eta^LD_Fq_F^t.
    \label{eq:smida-leader-oscillation-recursion}
  \end{equation*}
  Since $\pi_0=\pi^{\mathrm{ref}}$,
  \begin{equation*}
    \max_x\operatorname{osc}(m_0^x)=\frac{D_L}{\tau^L}.
  \end{equation*}
    Expanding Equation~\eqref{eq:smida-leader-oscillation-recursion} therefore gives
  \begin{equation*}
    \max_x\operatorname{osc}(m_t^x)
    \leq
    q_L^t\frac{D_L}{\tau^L}
    +\eta^LD_F
    \sum_{k=0}^{t-1}q_L^{t-1-k}q_F^k,
  \end{equation*}
  and the finite sum equals  $(q_L^t-q_F^t)/(q_L-q_F)$. Applying \eqref{eq:app:th-sda1} proves \eqref{eq:smida-leader-exact-bound}. Finally, $r(x,y)$ lies between $\min_{y'}u(x,y,y')$ and $\max_{y'}u(x,y,y')$, so $D_L\leq D$. Since $q_F<q_L$,
  \begin{equation*}
    \sum_{k=0}^{t-1}q_L^{t-1-k}q_F^k
    \leq\frac{q_L^t}{\eta^F\tau^F-\eta^L\tau^L},
  \end{equation*}
  which gives \eqref{eq:smida-leader-simple-bound}.
\end{proof}

Note that Theorem~\ref{thm:smida-main} follows by averaging the uniform bounds of Theorem~\ref{thm:smida-last-iterate}. Indeed, by definition,
\begin{equation*}
    \mathrm{KL}_{\rho}(\pi\|\widetilde\pi)
=
\sum_x \rho(x)\,
\mathrm{KL}\bigl(\pi(\cdot\mid x)\|
\widetilde\pi(\cdot\mid x)\bigr),
\end{equation*}
and
\begin{equation*}
    \mathrm{KL}_{\rho,\pi^*}(\omega\|\widetilde\omega)
=
\sum_{x,y}\rho(x)\pi^*(y\mid x)\,
\mathrm{KL}\bigl(\omega(\cdot\mid x,y)\|
\widetilde\omega(\cdot\mid x,y)\bigr).
\end{equation*}
The weights in each sum are nonnegative and sum to one. Thus, applying the uniform bounds in Theorem~\ref{thm:smida-last-iterate} separately to both directions of the KL divergence gives, for every $t \ge 1$,
\begin{equation*}
    \max\!\left\{
\mathrm{KL}_{\rho}(\pi^*\|\pi_t),
\mathrm{KL}_{\rho}(\pi_t\|\pi^*)
\right\}
\leq
\frac{D}{\tau^L}\frac{\kappa}{\kappa-1}
(1-\eta^L\tau^L)^t,
\end{equation*}
and
\begin{equation*}
    \max\!\left\{
\mathrm{KL}_{\rho,\pi^*}(\omega^*\|\omega_t),
\mathrm{KL}_{\rho,\pi^*}(\omega_t\|\omega^*)
\right\}
\leq
\frac{D_F}{\tau^F}
(1-\eta^F\tau^F)^t.
\end{equation*}
Consequently, the strictly positive constants required in the main theorem may be chosen as
\begin{equation*}
    C_L=\max\!\left\{1,\frac{D}{\tau^L}\frac{\kappa}{\kappa-1}\right\},
\qquad
C_F=\max\!\left\{1,\frac{D_F}{\tau^F}\right\}.
\end{equation*}
These constants are finite because the action and context sets are finite, $\tau^L,\tau^F>0$, and $\kappa>1$. The step-size condition ensures that both geometric factors lie in $[0,1)$, so the bounds imply last-iterate linear convergence; full support of every iterate is already established in Theorem~\ref{thm:smida-main}. 

\subsection{The Follower's Timescale and the Leader's Approximation Error}\label{subsec:appB1}

The choice of a faster Follower in \eqref{eq:step-sizes} follows the approach adopted in SLHF \citep{pasztor2026stackelberg}. The underlying intuition is standard in two-timescale optimization: the inner player approaches its best response while the outer player changes more slowly, allowing the latter's updates to approximate optimization of the maximized objective \citep{DBLP:journals/jmlr/Lin0J25}. In our tabular setting, the Follower's objective at each $(x,y)$ is independent of the Leader's policy. The faster timescale therefore serves to reduce the error in the conditional values $\widehat{r}_t$ used by the Leader.

The role of the Follower's approximation error can be made explicit. Let $r(x,y)\coloneq\tau^F\log Z^F(x,y)$ denote the Follower's optimal regularized value at $(x,y)$. The following lemma identifies the difference between this value and the current conditional value $\widehat{r}_t(x,y)$.

\begin{lemma}
Under the assumptions of Theorem~B.1, for every $t\geq 0$ and every $(x,y)\in\mathcal{X}\times\mathcal{Y}$,
\begin{equation*}
    r(x,y)-\widehat{r}_t(x,y)
=\tau^F\KL\left(
\omega_t(\cdot\mid x,y)\,\|\,
\omega^*(\cdot\mid x,y)
\right)
\geq 0.
\end{equation*}
\end{lemma}

\begin{proof}
Fix $(x,y)$. Apply the identity established in the proof of Lemma~\ref{lem:variational-minimization} with $h=-u(x,y,\cdot)$, reference distribution $\omegaref(\cdot\mid x,y)$, and regularization parameter $\tau^F$. The unique minimizer is $\omega^*(\cdot\mid x,y)$ and the minimum is $-r(x,y)$. Evaluating the identity \eqref{eq:appendix-variational} at $\omega_t(\cdot\mid x,y)$ therefore gives
\begin{equation*}
    -\widehat{r}_t(x,y)
=-r(x,y)+\tau^F\KL\left(
\omega_t(\cdot\mid x,y)\,\|\,
\omega^*(\cdot\mid x,y)
\right).
\end{equation*}

Rearranging proves the equality. Nonnegativity follows from that of KL divergence.
\end{proof}

Thus, convergence of the Follower directly controls the error in the costs used by the Leader.

\section{Deep-smida: Policy Gradients and Implementation Details}\label{app:deep}

This Appendix contains algorithmic and implementation details for \textbf{\textsc{deep-smida}} (see Algorithm~\ref{alg:deep-stackelberg}). We parametrize the Leader's and Follower's policies by $\pi_\theta$ and $\omega_\phi$, respectively. Denote by $\rho$ the full support distribution over contexts and by $\piref, \omegaref$ the corresponding reference policies.

Motivated by the objective in Equation~\eqref{eq:regularized-eq} and the procedure described in Algorithm~\ref{alg:smida}, we write the Follower's exact objective and Leader's plug-in objective, respectively, as
\begin{align}
J^F(\theta, \phi) &\coloneq \mathbb{E}_{x \sim \rho(\cdot), y \sim \pi_\theta(\cdot \mid x), y' \sim \omega_\phi(\cdot \mid x,y)}[u(x,y,y')]-\tau^F\KL_{\rho, \pi_\theta}(\omega_\phi \,\|\, \omegaref),
\label{eq:follower-objective}
\\
J^L(\theta)
&\coloneq
\mathbb{E}_{\substack{
    x \sim \rho(\cdot),\;
    y \sim \pi_\theta(\cdot \mid x)
}}
\left[\widehat{r}(x,y)\right]
+\tau^L
\operatorname{KL}_{\rho}
\left(\pi_\theta \,\middle\|\, \pi_{\mathrm{ref}}\right),
\label{eq:leader-objective}
\end{align}
where following Equation~\eqref{eq:smida-current-cost}, $\widehat{r}_\phi$ denotes the plug-in cost currently faced by the Leader:
\begin{equation*}
    \widehat{r}_\phi(x,y) \coloneq \mathbb{E}_{y'\sim\omega_\phi(\cdot \mid x,y)}[u(x,y,y')]-\tau^F\KL_{\rho, \pi_\theta}(\omega_\phi(\cdot \mid x,y) \,\|\, \omegaref(\cdot \mid x,y))
\end{equation*}
Note that the two players optimize in opposite directions, the follower maximizes $J^F$, whereas the leader minimizes $J^L$.

For any baselines $b^F(x,y)$ and $b^L(x)$ that do not depend on the sampled Follower and Leader actions, respectively, the policy gradients can be written as
\begin{align}
\nabla_\phi J^F
&=
\mathbb{E}_{x,y,y'}\!\left[
\nabla_\phi\log\omega_\phi(y'\mid x,y)
\left(
u(x,y,y')
-\tau^F\log\frac{\omega_\phi(y'\mid x,y)}
{\omegaref(y'\mid x,y)}
-b^F(x,y)
\right)
\right],
\label{eq:deep-follower-gradient}\\
\nabla_\theta J^L
&=
\mathbb{E}_{x,y}\!\left[
\nabla_\theta\log\pi_\theta(y\mid x)
\left(
\widehat{r}_\phi(x,y)
+\tau^L\log\frac{\pi_\theta(y\mid x)}
{\piref(y\mid x)}
-b^L(x)
\right)
\right].
\label{eq:deep-leader-gradient}
\end{align}
The constant terms produced by differentiating the KL regularizers can be absorbed into the baselines because the expected score of a policy is zero.

We next give unbiased Monte Carlo estimators of these gradients. Fix a context $x\sim\rho$. Draw $D\geq 2$ independent Leader samples $y^1,\ldots,y^D\sim\pi_\theta(\cdot\mid x)$ and, for every $l\in\{1,\ldots,D\}$, draw $M\geq 2$ independent Follower samples $y'^{l1},\ldots,y'^{lM}\sim\omega_\phi(\cdot\mid x,y^l)$. In practice, the resulting estimators are additionally averaged over a minibatch of $B$ contexts. Define
\begin{equation}\label{eq:algo1}
A^{lj} \coloneq
u(x,y^l,y'^{lj})
-\tau^F\log
\frac{\omega_\phi(y'^{lj}\mid x,y^l)}
{\omegaref(y'^{lj}\mid x,y^l)},
\qquad
\overline A^l:=\frac{1}{M}\sum_{j=1}^M A^{lj}.
\end{equation}
For each Follower sample, use the leave-one-out \citep{DBLP:conf/iclr/KoolHW19a} baseline $b^F_{-j}(x,y^l):=(M-1)^{-1}\sum_{k\neq j}A^{lk}$. The Follower gradient estimator is then
\begin{equation}\label{eq:deep-follower-estimator}
\widehat{\nabla_\phi J^F} \coloneq
\frac{1}{D}\sum_{l=1}^D\frac{1}{M}\sum_{j=1}^M
\nabla_\phi\log\omega_\phi(y'^{lj}\mid x,y^l)
\left(A^{lj}-b^F_{-j}(x,y^l)\right).
\end{equation}
Conditionally on $(x,y^l)$, the baseline $b^F_{-j}(x,y^l)$ is independent of $y'^{lj}$. It therefore reduces variance without changing the expectation of the estimator. The same rollouts also provide the Leader estimator; no additional sampling is required. In particular, $\overline A^l$ is a conditionally unbiased estimate of $\widehat r(x,y^l)$. Define
\begin{equation}\label{eq:algo2}
    L^l \coloneq
\overline A^l
+\tau^L\log\frac{\pi_\theta(y^l\mid x)}
{\pi^{\mathrm{ref}}(y^l\mid x)},
\qquad
b^L_{-l}(x):=\frac{1}{D-1}\sum_{m\neq l}L^m.
\end{equation}
The resulting Leader gradient estimator is
\begin{equation}\label{eq:deep-leader-estimator1}
\widehat{\nabla_\theta J^L} \coloneq
\frac{1}{D}\sum_{l=1}^D
\nabla_\theta\log\pi_\theta(y^l\mid x)
\left(L^l-b^L_{-l}(x)\right).
\end{equation}
Here $b^L_{-l}(x)$ is independent of $y^l$ conditional on $x$, so the Leader estimator is also unbiased. Averaging Equations~\eqref{eq:deep-follower-estimator} and~\eqref{eq:deep-leader-estimator1} over a minibatch of contexts preserves this property.

At iteration $t$, the Follower takes an ascent step and the Leader a descent step:
\begin{equation}
\phi_{t+1}
=\phi_t+\eta_t^F\widehat{\nabla_\phi J^F},
\qquad
\theta_{t+1}
=\theta_t-\eta_t^L\widehat{\nabla_\theta J^L}.
\label{eq:deep-updates}
\end{equation}
In practice we apply these updates through AdamW optimizers \citep{DBLP:conf/iclr/LoshchilovH19} with $\eta^L_t$ and $\eta^F_t$ as their learning rates and no weight decay; each gradient estimate is clipped to unit norm before the optimizer step. Both policies are initialized from their corresponding reference policies, which remain frozen throughout training. Algorithm~\ref{alg:deep-stackelberg} summarizes the complete procedure: for every minibatch, it samples the nested Leader--Follower rollouts, constructs both gradient estimators, and applies the two updates in Equation~\eqref{eq:deep-updates}.

\subsection{Learning Rate Warm-up.}\label{subsec:learn-warmup}

As discussed in Section~\ref{sec:deep-impl-maintext}, directly starting from the target two-timescale learning rates can make training unstable. Let $T_{\mathrm{warm}}$ be the number of warm-up iterations. We linearly warm the Leader's rate from $\eta^L_{\mathrm{init}}$ to its target value $\eta^L$ and, simultaneously, the ratio of the Follower and Leader rates from $1$ to its target value $\eta^F/\eta^L$:
\begin{align*}
\eta_t^L
&=
\eta^L_{\mathrm{init}}
+\min\!\left\{\frac{t}{T_{\mathrm{warm}}},1\right\}
\left(\eta^L-\eta^L_{\mathrm{init}}\right),\\
\eta_t^F
&=
\eta_t^L
\left[
1+\min\!\left\{\frac{t}{T_{\mathrm{warm}}},1\right\}
\left(\frac{\eta^F}{\eta^L}-1\right)
\right].
\end{align*}
Thus, the two rates are equal at initialization, while after warm-up they equal their target values $\eta^L$ and $\eta^F$. We choose the targets so that $\eta^L\tau^L<\eta^F\tau^F$, matching the two-timescale ordering of Algorithm~\ref{alg:smida}. 

\begin{algorithm}[t]
\caption{Deep Stackelberg Mirror Descent--Ascent (\textbf{\textsc{deep-smida}})}
\label{alg:deep-stackelberg}
\begin{algorithmic}[1]
\Require Utility $u$; context distribution $\rho$; reference policies $\piref$ and $\omegaref$; regularization weights $\tau^L,\tau^F>0$; minibatch size $B$; sample sizes $D,M\geq2$; step sizes $\{\eta_t^L,\eta_t^F\}_{t\geq0}$
\State Initialize $\theta_0$ and $\phi_0$ from their corresponding reference policies, which remain frozen throughout training
\For{$t=0,1,2,\ldots$}
    \State Draw a minibatch
    $\mathcal{B}_t=\{x_i\}_{i=1}^B$ with
    $x_i\sim\rho$
    \For{$i=1,\ldots,B$}
        \State Draw $D$ independent Leader samples
        $y_i^1,\ldots,y_i^D
        \sim\pi_{\theta_t}(\cdot\mid x_i)$
        \For{$\ell=1,\ldots,D$}
            \State Draw $M$ independent Follower samples
            $y_i^{\prime\,\ell 1},\ldots,y_i^{\prime\,\ell M}
            \sim\omega_{\phi_t}(\cdot\mid x_i,y_i^\ell)$
            \State Compute $A_i^{\ell j}$, $\overline{A}_i^\ell$, and
            $b^F_{-j}(x_i,y_i^\ell)$ according to
            \eqref{eq:algo1}
        \EndFor
        \State Form the per-context Follower gradient estimator
        $\widehat{\nabla_\phi J^F}(x_i)$ according to
        \eqref{eq:deep-follower-estimator}
        \State Using the same rollouts, compute $L_i^\ell$ and
        $b^L_{-\ell}(x_i)$ according to
        \eqref{eq:algo2}
        \State Form the per-context Leader gradient estimator
        $\widehat{\nabla_\theta J^L}(x_i)$ according to
        \eqref{eq:deep-leader-estimator}
    \EndFor
    \State Average the per-context estimators over $\mathcal{B}_t$ to
    obtain $\widehat{\nabla_\phi J^F}_t$ and
    $\widehat{\nabla_\theta J^L}_t$
    \State Update the Follower by ascent and the Leader by descent as in
    \eqref{eq:deep-updates}:
    \vspace{-0.3\baselineskip}
    \[
    \phi_{t+1}
    =
    \phi_t+\eta_t^F\widehat{\nabla_\phi J^F}_t,
    \qquad
    \theta_{t+1}
    =
    \theta_t-\eta_t^L\widehat{\nabla_\theta J^L}_t.
    \]
    \vspace{-\baselineskip}
\EndFor
\end{algorithmic}
\end{algorithm}

\subsection{Computational Complexity}

We provide a brief analysis of the computational complexity of Algorithm~\ref{alg:deep-stackelberg}. Each iteration of Algorithm 2 samples $B$ contexts \(\{x_i\}_{i \in \{1,2,\dots,B\}}\), $D$ Leader responses \(\{y_i^\ell\}_{l \in \{1,\dots,D\}}\) per context $x_i$, and $M$ Follower responses \(\{y_i'^{\ell d}\}_{d \in \{1,\cdots,M\}}\) per Leader response \(y_i^\ell\). It therefore generates $BD$ Leader responses and $BDM$ Follower responses, for a total of $BDM$ response pairs. From \eqref{eq:deep-follower-estimator}, we observe that estimating \(\widehat{\nabla_\phi J^F}\) requires a forward-backward pass through \(DM\) response trajectories and also \(DM\) forward passes from the frozen reference. We, notice however that the \(D\) rollouts of the leader can be reused. Hence, from \label{eq:deep-leader-estimator}, estimating \(\widehat{\nabla_\theta J^L}\), only requires \(D\) forward-backward passes through the leader and\(D\) forward passes from the frozen reference. The number of \emph{NFE} or Neural Function Evaluations is slightly higher for the follower, due to its higher context length. Further, each of the \(DM\) response pairs also requires evaluations from each of the \(k=5\) preference ensemble. Multiplying this by the mini-batch size, that for each optimizer step, we require \(BDMk+3BDM+3BD\) rollouts.

\section{Experimental Setup and Implementation Details}\label{app:implementation}

This appendix provides details of the experimental setup used in Section~\ref{sec:experiments}. We describe dataset construction and oracle labelling, the architecture and training of the preference ensemble, the distribution of ensemble variance, and the implementation of the Leader--Follower game.

\subsection{Dataset}\label{app:dataset}

We use a subset of UltraFeedback \citep{DBLP:conf/icml/CuiY0YH0NXXL0024}, a preference dataset containing approximately $64$k prompts collected from six public instruction sources, each answered by four language models drawn from a pool of seventeen. We retain the prompts and their responses, discard empty and duplicate responses, and expand each prompt $x$ into all $\binom{N_x}{2}$ unordered pairs of its $N_x$ distinct responses. Each pair defines a preference triplet $(x,y,y')$.

Following \citet{DBLP:conf/icml/GaoSH23}, we adopt a synthetic oracle setup, using \texttt{DeepSeek-V4-Flash} \citep{deepseekai2026deepseekv4} as the gold preference oracle. We replace UltraFeedback's original annotations with oracle labels so that the original response pairs and comparisons subsequently generated by the adversarial game share a common source of supervision. For each comparison, we obtain a binary preference label through zero-shot prompting, randomizing the display order of the two responses to mitigate position bias. The procedure gives a dataset $\mathcal{D}$ of $89{,}625$ labelled preference triplets across $15{,}053$ unique prompts.

\subsection{Preference Models}\label{app:preference-models}

We train an ensemble of $|I|=5$ preference models $\{\mathcal{P}_{\theta_i}\}_{i \in I}$, using a simplified version of the general preference model of \citet{DBLP:conf/icml/ZhangZWXG25}. Each member is initialized from the same \texttt{Qwen2.5-1.5B-Instruct} backbone \citep{qwen2024qwen25} and finetuned independently. Given a preference triplet $(x,y,y')$, we encode each prompt-response pair $(x,y), (x,y')$ separately, project its final-token hidden representation through a bias-free linear head into a $16$-dimensional vector, and apply $L^2$ normalization, obtaining an embedding $z_{\theta_i}(x,y)$. We compare the resulting embeddings through a fixed skew-symmetric bilinear form $J \in \mathbb{R}^{16 \times 16}$, rescale the score by a learned positive scalar $\lambda_{\theta_i}$ shared across inputs, and apply a sigmoid $\sigma$ to obtain the preference probability. The preference probability is
\begin{equation}
\mathcal{P}_{\theta_i}(y \succ y' \mid x)
\coloneq
\sigma\big(
\lambda_{\theta_i}
\langle J z_{\theta_i}(x,y), z_{\theta_i}(x,y')\rangle
\big).
\end{equation}
This construction ensures that $\mathcal{P}_{\theta_i}(y \succ y' \mid x) + \mathcal{P}_{\theta_i}(y' \succ y \mid x)=1$, while allowing intransitive preferences. By the antisymmetry of the preference models and basic properties of the variance, we have $\operatorname{Var}(x,y,y') = \operatorname{Var}(x,y',y)$; thus, the order of the two responses within a preference triplet does not affect the optimization objective.

Each ensemble member is trained by minimizing the negative log-likelihood $-\log\mathcal{P}_{\theta_i}(y_w\succ y_l\mid x)$ of the oracle labels on its own training partition of $\mathcal{D}$, where $y_w$ and $y_l$ denote the oracle-preferred and non-preferred responses. Each partition contains approximately $10{,}000$ preference triplets across $1{,}680$ unique prompts. The partitions are mutually disjoint at the prompt level and all labelled comparisons associated with a training prompt belong to the same partition. Ensemble diversity comes from these distinct training sets, independent random initializations of the linear heads, and different training seeds.

All members share a validation set whose prompts are disjoint from every training partition. A further $37{,}614$ triplets across $6{,}315$ prompts are excluded from both training and validation and serve as the held out set of Sections~\ref{subsec:variance-experiments}. On this set, individual members agree with the gold preference oracle on $74.1$--$75.2\%$ of comparisons, while predictions obtained from the ensemble-mean preference probability agree on $75.7\%$. 

\subsection{Distribution of Ensemble Variance}\label{app:variance-distribution}

We examine the distribution of $\operatorname{Var}(x,y,y')$ on the held out UltraFeedback subset, comprising $37{,}614$ response pairs across $6{,}315$ prompts. For each pair, we compute the population variance of the five members' predicted preference probabilities, using denominator $|I|=5$. All descriptive statistics weight response pairs equally.

Table~\ref{tab:variance-summary} summarizes the distribution. Most comparisons elicit little ensemble disagreement, while a smaller fraction forms a pronounced right tail: the mean is more than twice the median, and the $99$th percentile is approximately fifteen times the median. Only $8.93\%$ of comparisons have variance greater than $0.02$. Consistent with the relationship shown in the main text, mean variance is higher on comparisons where the ensemble disagrees with the synthetic oracle ($0.010010$) than where it agrees ($0.005687$). These statistics describe disagreement within the trained ensemble and its association with oracle error on this held out subset.

\begin{table}[t]
\centering
\caption{Distribution of ensemble variance on the held out UltraFeedback subset.}
\vspace{1pt}
\label{tab:variance-summary}
\begingroup
\setlength{\tabcolsep}{10pt}
\renewcommand{\arraystretch}{1.15}
\begin{tabular}{@{}lr@{\hspace{2.5em}}lr@{}}
\toprule
\multicolumn{2}{c}{Summary statistics} &
\multicolumn{2}{c}{Selected percentiles} \\
\cmidrule(lr){1-2}\cmidrule(lr){3-4}
Statistic & Value & Percentile & Value \\
\midrule
Mean               & $0.006739$          & $25$ & $0.000588$ \\
Median             & $0.002931$          & $75$ & $0.008959$ \\
Variance           & $9.0507\times10^{-5}$ & $90$ & $0.018687$ \\
Standard deviation & $0.009514$          & $95$ & $0.026660$ \\
Maximum            & $0.100228$          & $99$ & $0.044667$ \\
\bottomrule
\end{tabular}
\endgroup
\end{table}

\subsection{Stackelberg Game Implementation}\label{app:game-implementation}

The Leader and Follower are both initialized from \texttt{Qwen2.5-1.5B-Instruct} \citep{qwen2024qwen25} and finetuned during the Stackelberg game. Their frozen initializations serve as the reference policies $\pi^{\mathrm{ref}}$ and $\omega^{\mathrm{ref}}$. 

Table~\ref{tab:game-hyperparameters} summarizes the training configuration. The utility is $4\operatorname{Var}(x,y,y')$, where variance is computed across the five frozen ensemble members with denominator $5$, giving a score in $[0,1]$.

\begin{table}[t]
    \centering
    \caption{Hyperparameters for training the Leader and Follower. Unless specified otherwise, settings apply to both policies.}
    \label{tab:game-hyperparameters}
    \small
    \renewcommand{\arraystretch}{1.08}
    \begin{tabular}{@{}lr@{}}
        \toprule
        \textbf{Parameter} & \textbf{Value} \\
        \midrule
        \multicolumn{2}{@{}l}{\textit{Optimization}} \\
        Optimizer & AdamW \\
        Leader learning rate $\eta^L$ & $8\times 10^{-6}$ \\
        Follower learning rate $\eta^F$ & $2\times 10^{-5}$ \\
        Initial learning rate (both policies) & $10^{-6}$ \\
        Warm-up steps & $50$ \\
        KL coefficients $\tau^L,\tau^F$ & $5\times 10^{-4}$ \\
        Weight decay & $0$ \\
        Maximum gradient norm & $1.0$ \\
        Parameter precision & \texttt{bfloat16} \\
        \addlinespace
        \multicolumn{2}{@{}l}{\textit{Training and rollout sampling}} \\
        Training prompts & $2{,}000$ \\
        Training epochs & $1$ \\
        Optimizer steps & $500$ \\
        Prompts per step $B$ & $4$ \\
        Leader responses per prompt $D$ & $4$ \\
        Follower responses per Leader response $M$ & $4$ \\
        \addlinespace
        \multicolumn{2}{@{}l}{\textit{Generation}} \\
        Sampling temperature & $0.6$ \\
        Top-$p$ & $1.0$ \\
        Maximum generated tokens per response & $1{,}500$ \\
        \bottomrule
    \end{tabular}
\end{table}

During warm-up, we linearly increase the Leader's learning rate from $10^{-6}$ to $\eta^L$ and the Follower-to-Leader learning-rate ratio from $1$ to $\eta^F/\eta^L=2.5$. The Follower's learning rate is their product; both rates remain constant after warm-up. Each optimizer step uses $BDM=64$ preference triplets. Sampling and policy log-probability calculations use the same temperature, with no nucleus or top-$k$ truncation.

We use the following chat templates, with the Follower receiving the Leader's realized response in its system message.

\begin{promptbox}[colback=promptultramarine,colframe=promptultramarineborder]{Leader Template}
[System]: You are Qwen, created by Alibaba Cloud. You are a helpful assistant.

[User]: <USER_PROMPT>
\end{promptbox}

\begin{promptbox}[colback=promptred,colframe=promptredborder]{Follower Template}
[System] You are Qwen, created by Alibaba Cloud. You are a helpful assistant.

Another assistant previously answered the user's request as follows:

<LEADER_RESPONSE>

Write a different, higher-quality answer to the user's request. Respond directly to the user as if answering for the first time. Do not mention, quote, or reference the previous answer.

[User]: <USER_PROMPT>

\end{promptbox}

\section{The Utility as a Design Choice}\label{sec:utility-as-design}

Our framework accommodates any fixed utility $u$ without changing the game structure or update rules, and the tabular guarantees remain valid under the stated assumptions. In our experiments, $u$ is the variance of the preference probabilities predicted by a trained ensemble: the Follower seeks to expose ensemble disagreement, while the Leader seeks to limit it. This is not the only option and exploring alternative utilities is a natural direction for future work.

Preference optimization is another special case. Setting $u(x,y,y')=\mathcal{P}(y'\succ y\mid x)$ recovers an equivalent minimization form of the regularized SLHF \cite{pasztor2026stackelberg} objective. Our analysis thus establishes linear last-iterate convergence of \textbf{\textsc{smida}} for both players in this preference game. In the finite tabular setting with strictly positive regularization and the stated step-size conditions, this answers affirmatively the question raised in \citet{pasztor2026stackelberg} of whether mirror-descent methods can be adapted to obtain last-iterate convergence for Stackelberg preference games.

\subsection{Other Uncertainty Metrics}

Further examples of utility are motivated by uncertainty measures used for query selection in active preference learning, such as ensemble variance and predictive entropy \citep{DBLP:conf/nips/ChristianoLBMLA17,DBLP:conf/icml/LeeSA21}. We consider here two alternatives to variance: decision uncertainty \citep{Lewis1994ASA} and Jensen--Shannon divergence (JSD) \citep{Lin1991DivergenceMB}. Writing $\mathcal{P}_i(x,y,y')$ for the preference probability predicted by ensemble member $i$, let
$\bar{\mathcal{P}}(x,y,y') = \frac{1}{|I|}\sum_{i \in I}\mathcal{P}_i(x,y,y')$
denote the ensemble mean. We define decision uncertainty as
\begin{equation*}
U(x,y,y') \coloneq
\min\left\{\bar{\mathcal{P}}(x,y,y'),\,1-\bar{\mathcal{P}}(x,y,y')\right\}
=
\frac{1}{2}-\left|\bar{\mathcal{P}}(x,y,y')-\frac{1}{2}\right|,
\end{equation*}
and JSD as
\begin{equation*}
\operatorname{JSD}(x,y,y') \coloneq
H\!\left(\bar{\mathcal{P}}(x,y,y')\right)
-\frac{1}{|I|}\sum_{i \in I}H\!\left(\mathcal{P}_i(x,y,y')\right),
\end{equation*}
where $H(p)=-p\log p-(1-p)\log(1-p)$ denotes the entropy of a Bernoulli distribution with parameter $p$. Decision uncertainty measures the proximity of the ensemble mean to indifference and induces the same ranking as its predictive entropy, $H(\bar{\mathcal{P}})$. JSD instead measures disagreement among the predictive distributions of the individual members.

We evaluate both scores on the held-out comparisons of Section~\ref{subsec:variance-experiments}, measuring errors relative to the binary oracle labels. Figure~\ref{fig:uncertainty-jsd-comparison}(a) reports error-detection AUROCs of $0.707$ for decision uncertainty and $0.654$ for JSD. Selecting the top $10\%$ of comparisons ranked by each score gives error rates of $47.5\%$ and $39.7\%$, respectively, compared with $24.3\%$ overall (Figure~\ref{fig:uncertainty-jsd-comparison}(b)). On this evaluation, decision uncertainty therefore outperforms variance in both AUROC ($0.670$) and error concentration in the top $10\%$ ($40.6\%$), while JSD gives slightly lower values on both measures.

Better error detection on a fixed set of comparisons, however, does not necessarily make a score a suitable utility for the adversarial game. Decision uncertainty is maximized when the ensemble mean is $1/2$, including when all members agree on indifference. Since the Follower observes the Leader's realized response, it can attain this maximum by copying that response, setting $y'=y$. The antisymmetry of our preference models ensures that $\mathcal{P}_i(x,y,y)=1/2$ for every member, and hence $U(x,y,y)=1/2$ for every $(x,y)$. Thus, copying is a best response to every Leader response, and the Leader cannot reduce the Follower's optimal utility below $1/2$. Variance and JSD avoid this particular failure mode, since both vanish when all ensemble members predict $1/2$, including on identical-response comparisons. 

\begin{figure}[t]
\centering
\includegraphics[width=\linewidth]{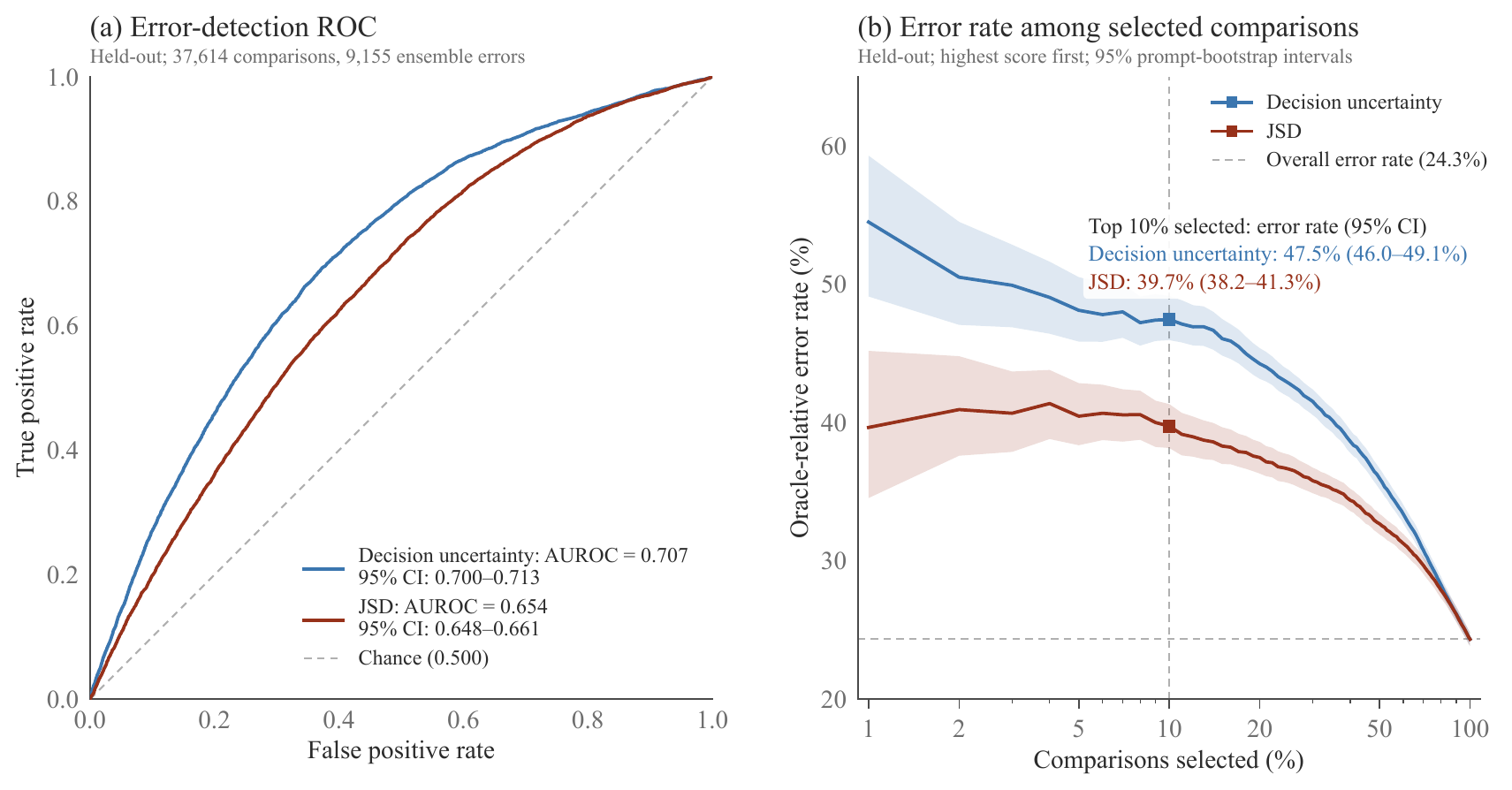}
\caption{Preference-ensemble error detection on $37{,}614$ held-out comparisons, including $9{,}155$ errors relative to the binary oracle labels.
\textbf{(a)} ROC curves for decision uncertainty (blue) and Jensen--Shannon divergence (JSD; red), with AUROCs of $0.707$ and $0.654$, respectively. The dashed diagonal indicates chance performance.
\textbf{(b)} Oracle-relative error rates among comparisons selected in descending order of each score. Selecting the top $10\%$ gives error rates of $47.5\%$ for decision uncertainty ($95\%$ CI: $46.0$--$49.1\%$) and $39.7\%$ for JSD ($95\%$ CI: $38.2$--$41.3\%$), compared with $24.3\%$ overall. Shading shows pointwise $95\%$ prompt-bootstrap intervals.}
\label{fig:uncertainty-jsd-comparison}
\end{figure}

\subsection{Targeting a Specific Preference Model}

Another option is to augment the ensemble disagreement objective in Equation~\eqref{eq:first-adv-game} with a preference term. Given a fixed member $\mathcal{P}_{\theta_j}$ of the ensemble and a coefficient $\lambda \geq 0$, consider
\begin{equation*}
\min_{\pi}\max_{\omega}
\mathbb{E}_{x\sim\rho,\,y\sim\pi(\cdot\mid x),\,y'\sim\omega(\cdot\mid x,y)}
\left[
\operatorname{Var}(x,y,y')
+\lambda\,\mathcal{P}_{\theta_j}(y'\succ y\mid x)
\right].
\end{equation*}
The additional term serves two purposes. First, it acts as a preference-based regularizer as the Follower is encouraged to generate responses that the selected model prefers to the Leader's response. The Leader, in turn, seeks responses that resist both disagreement and defeat under this model. This introduces a preference signal for optimization, with $\lambda$ controlling its strength relative to variance. Second, the term allows us to direct the search toward a specific preference model while retaining uncertainty information from the full ensemble. Experiments show that this directed objective achieves comparable performance to the variance objective in generating $\alpha$-online candidates.

\section{Further Experiments, Parameters Sweep, and Ablation}\label{app:further-exp}

\subsection{Training Dynamics}

Figure~\ref{fig:variance-cdf-and-online-candidates} complements the evaluation in Section~5.3. Panel~(a) shows the empirical CDF of ensemble variance on held-out prompts, computed from generations obtained after training. The trained pair's distribution shifts toward higher variance relative to the reference pair; vertical lines mark the reference-quantile thresholds defining $\alpha$-online candidates.

Panel~(b) tracks candidate shares in training minibatches of $64$ triplets over $500$ steps, using these same fixed thresholds. The moving averages begin near the nominal reference rates, rise markedly after approximately $200$ steps, and remain elevated through the end of training despite substantial minibatch variability. By the final steps, the smoothed shares are approximately $50\%$ at $\alpha=0.1$ and $20\%$ at $\alpha=0.01$. The increased candidate yield is thus sustained during training and is also observed on held-out prompts after training. 

For examples of high variance generated preference triplets, see Appendix~\ref{app:example-responses}.

\begin{figure}[t]
    \centering
    \includegraphics[width=\linewidth]{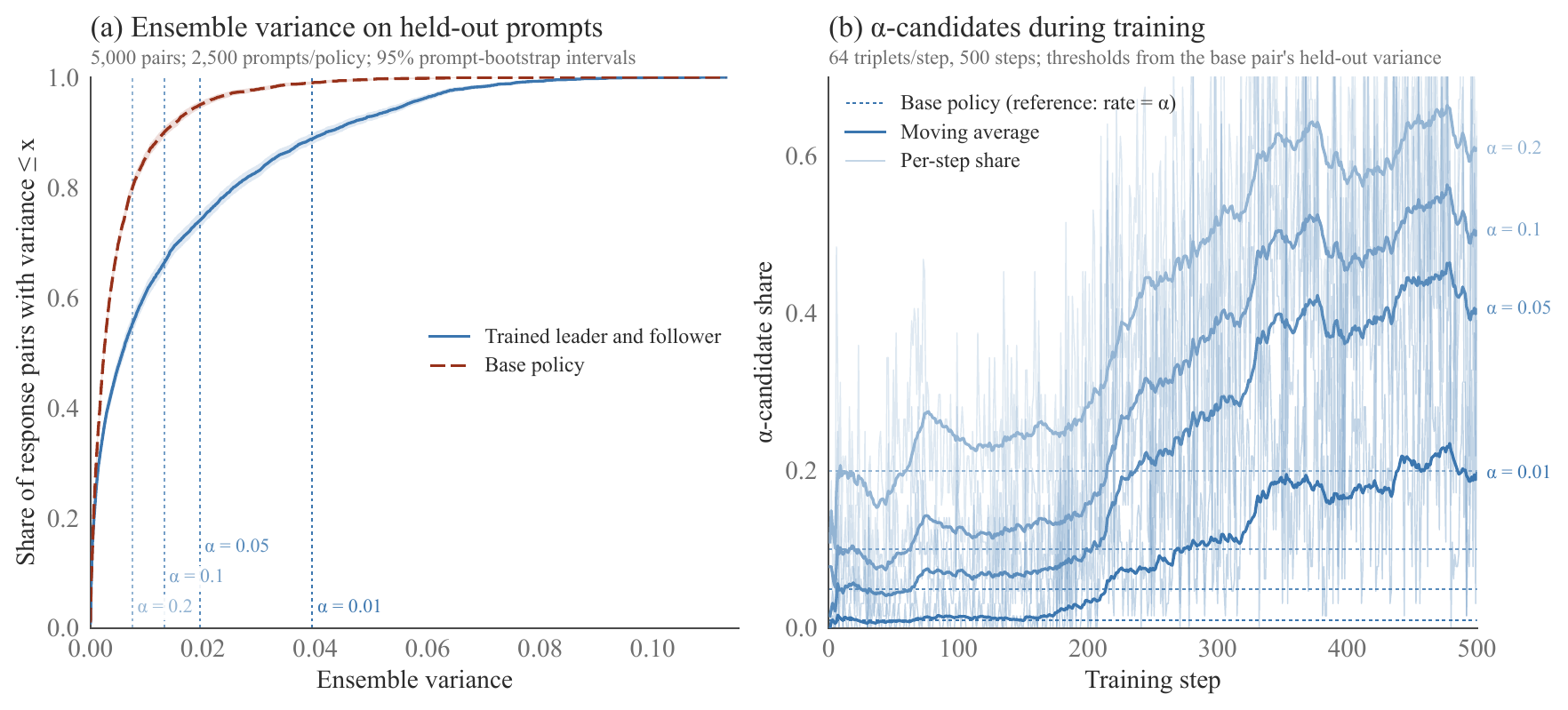}
    \caption{Ensemble disagreement after and during training.
    \textbf{(a)} Empirical CDFs of ensemble variance on $5{,}000$ triplets from $2{,}500$ held-out prompts per policy pair, generated after training. Vertical lines mark reference-quantile thresholds; shading shows $95\%$ prompt-bootstrap intervals.
    \textbf{(b)} Shares of $\alpha$-online candidates in training minibatches of $64$ triplets, using the same fixed thresholds. Thin curves show per-step shares, thick curves their moving averages, and horizontal dotted lines the nominal reference rates $\alpha$.}
    \label{fig:variance-cdf-and-online-candidates}
\end{figure}

\subsection{Hyperparameter Sweep}
\label{app:hyperparameter_sweep}

We performed a sequential hyperparameter sweep over the KL regularization coefficients, the Follower's target learning rate, and the warm-up duration. We first varied the common regularization coefficient
\begin{equation*}
    \tau = \tau^L = \tau^F
\in \{2\times10^{-4},\,5\times10^{-4},\,10^{-3},\,2\times10^{-3},\,4\times10^{-3}\}.
\end{equation*}
After selecting $\tau$, we swept the Follower's target learning rate
$\eta^F \in \{2\times10^{-5},\,3\times10^{-5},\,5\times10^{-5}\}$
and the warm-up duration $T_{\mathrm{warm}} \in \{15,20,50\}$ optimizer steps, using the joint learning-rate and learning-rate-ratio schedule described in Appendix~\ref{subsec:learn-warmup}.

Configuration selection considered both training stability and the fraction of generated triplets satisfying the $\alpha$-online candidate criterion of Definition~\ref{def:alpha-online-adv}.

\subsection{Adversarial Precision of Online Candidates}\label{app:candidate-precision}

We further examine whether the online criterion (Definition~\ref{def:alpha-online-adv}) identifies comparisons on which the ensemble's mean preference is incorrect. Using the trained policies' 1,600 oracle-labelled triplets from 800 holdout prompts described in Section~\ref{subsec:main-text-adver-exp}, Figure~\ref{fig:candidate-precision} reports the fraction of $\alpha$-online candidates that are $\varepsilon$-adversarial. Smaller $\alpha$ imposes a stricter disagreement threshold and generally increases this fraction. At $\alpha=0.1$, $29.7\%$ of candidates reverse the oracle preference ($\varepsilon=0$), compared with $19.4\%$ of all labeled triplets, a $1.53$-fold enrichment over random selection. The same pattern continues for $\varepsilon\in\{0.1,0.2,0.3\}$, indicating that the enrichment extends to reversals with larger oracle margins. These results support ensemble disagreement as an informative criterion for prioritizing comparisons for annotation within the trained policies' generations.

\begin{figure}[t]
    \centering
    \includegraphics[width=0.78\linewidth]
        {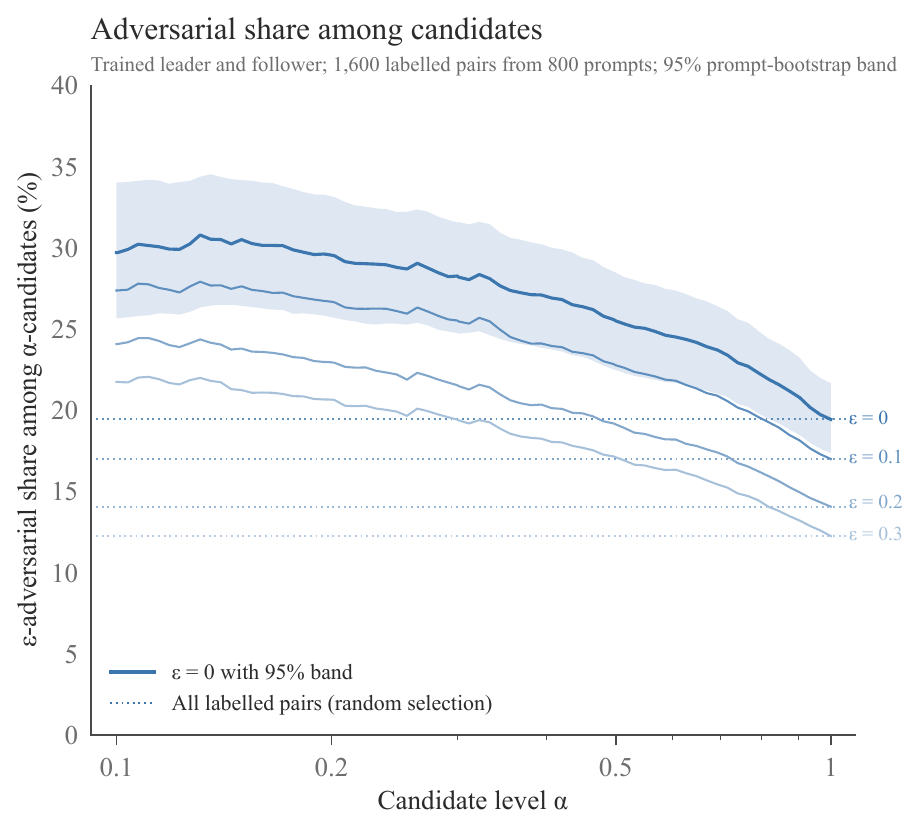}
    \caption{
        Adversarial precision of online candidates.
        Solid curves show the fraction of $\alpha$-online candidates that are $\varepsilon$-adversarial. Candidate thresholds are fixed by the
        reference variance distribution, as in Section~\ref{sec:experiments}
        Lighter curves correspond to larger oracle margins.
        Horizontal dotted lines show the corresponding adversarial fractions among all labelled triplets, giving the expected
        precision under random selection.
        Shading shows a 95\% prompt-bootstrap interval for
        $\varepsilon=0$.
    }
    \label{fig:candidate-precision}
\end{figure}

\subsection{Preference Evaluation of the Trained Leader}
\label{app:roundrobin}

We evaluate whether our Leader produces responses that remain competitive under preference evaluation, despite being only trained to minimize ensemble disagreement rather than maximize preference. We compare it with NLHF~\cite{DBLP:conf/icml/MunosVCARGTGMFM24} and the SLHF Leader~\cite{pasztor2026stackelberg} in a round-robin evaluation. All three policies are initialized from \texttt{Qwen2.5-1.5B-Instruct} and trained on the same $2{,}000$ prompts. NLHF and SLHF optimize preferences under the frozen model $\mathcal{P}_0$, whereas our Leader optimizes the ensemble-variance objective. To control for differences in the choice of trainable parameters, we use full-model fine-tuning for all three methods, with LoRA disabled. Table~\ref{tab:baseline-hyperparameters} reports the training and generation settings for the two baselines. We evaluate all three policy pairings on a separate set of $500$ held-out prompts, using \texttt{DeepSeek-V4-Flash} as the synthetic preference oracle. For each response pair, we query the oracle $k=11$ times and estimate the preference probability by the fraction of judgments favoring one response over the other. Table~\ref{tab:leader-pairwise} reports these win rates, averaged across evaluation prompts.

Both Leaders are preferred to NLHF more often than not: the SLHF Leader achieves a preference rate of $54.92\%$, and our Leader achieves $54.89\%$. Their direct comparison is nearly even, with our Leader preferred in $50.70\%$ of oracle judgments. 

This competitiveness is an unintended benefit of training our Leader to resist ensemble disagreement. We conjecture that minimizing disagreement against an adaptive Follower, together with regularization toward the initial instruction-tuned policy, may favor responses that retain useful instruction-following behavior. 

\begin{table}[t]
    \centering
    \caption{Training and generation settings for the NLHF
    (Nash-MD-PG) and SLHF (StackelbergGDA) baselines.
    Both methods fully fine-tune \texttt{Qwen2.5-1.5B-Instruct}
    on the same $2{,}000$ training prompts using the frozen
    preference model $\mathcal{P}_0$. Each uses a separate
    frozen copy of the initial policy as its KL reference.}
    \label{tab:baseline-hyperparameters}
    \small
    \setlength{\tabcolsep}{6pt}
    \renewcommand{\arraystretch}{1.08}
    \begin{tabular}{@{}lrr@{}}
        \toprule
        \textbf{Parameter}
            & \textsc{NLHF}
            & \textsc{SLHF} \\
        \midrule
        \multicolumn{3}{@{}l}{\textit{Training and rollout sampling}} \\
        Training prompts
            & $2{,}000$ & $2{,}000$ \\
        Training epochs
            & $1$ & $1$ \\
        Optimizer steps
            & $63$ & $63$ \\
        Prompts per optimizer step
            & $32$ & $32$ \\
        Prompts in final step
            & $16$ & $16$ \\
        Responses per prompt
            & $2$ & $2$ \\
        \addlinespace
        \multicolumn{3}{@{}l}{\textit{Generation}} \\
        Sampling temperature
            & $0.6$ & $0.6$ \\
        Top-$p$
            & $1.0$ & $1.0$ \\
        Maximum generated tokens per response
            & $512$ & $512$ \\
        \bottomrule
    \end{tabular}
\end{table}

\begin{table}[t]
    \centering
    \caption{Pairwise oracle preference rates on $500$ held-out
    prompts. Each entry estimates the probability that the
    row policy's response is preferred to the
    column policy's response, using $11$ oracle queries
    per response pair and averaging across prompts.
    Values above $0.5$ favor the row policy.}
    \label{tab:leader-pairwise}
    \small
    \setlength{\tabcolsep}{8pt}
    \renewcommand{\arraystretch}{1.25}
    \begin{tabular}{@{}lccc@{}}
        \toprule
        & \multicolumn{3}{c}{\textbf{Column policy}} \\
        \cmidrule(l){2-4}
        \textbf{Row policy}
            & \textsc{NLHF}
            & \textsc{SLHF Leader}
            & \textsc{Our Leader} \\
        \midrule
        \textsc{NLHF}
            & --- & $0.4508$ & $0.4511$ \\
        \textsc{SLHF Leader}
            & $\mathbf{0.5492}$ & --- & $0.4930$ \\
        \textsc{Our Leader}
            & $\mathbf{0.5489}$ & $\mathbf{0.5070}$ & --- \\
        \bottomrule
    \end{tabular}
\end{table}

\subsection{Instruction Following with a Trained Leader}
\label{app:ifeval}

We examine the role of the Leader by comparing the Follower trained in the full Stackelberg game with a follower-only baseline, in which the Leader remains fixed. Both Followers seek to maximize ensemble disagreement, but only the former trains against a Leader that learns to resist its attempts. This ablation evaluates whether adapting both sides of the comparison is associated with better instruction following after adversarial training.

We evaluate both Followers on IFEval~\cite{zhou2023instruction}, using the same $541$ prompts containing $834$ verifiable instructions. IFEval uses programmatic checks to assess compliance with explicit requirements, such as including specified keywords, respecting length constraints, or ending with a prescribed phrase. Instruction-level accuracy measures the fraction of individual requirements satisfied, whereas prompt-level accuracy requires all verifiable instructions in a prompt to be satisfied. We report both strict and loose accuracy. Strict evaluation checks the original response; loose evaluation additionally accepts specified transformations, such as removing Markdown asterisks or the first or last line, to reduce false negatives caused by formatting or introductory and concluding text.

Table~\ref{tab:ifeval-lr2-follower-only} shows that the Stackelberg Follower achieves higher accuracy on all four aggregate metrics. Loose prompt-level accuracy is $35.12\%$, compared with $28.28\%$ for the follower-only baseline, and loose instruction-level accuracy is $46.04\%$, compared with $38.49\%$. The advantage persists under strict evaluation and also extends across eight of the nine instruction categories. 

These results support our motivation for making the search adversarial on both sides of the comparison. With a fixed Leader, the Follower can specialize in inducing disagreement against an unchanging response distribution, potentially exploiting weaknesses specific to that distribution. An adaptive Leader instead learns responses that resist these attempts, requiring the Follower to search against an evolving opponent. The higher instruction-following accuracy is consistent with this interaction discouraging strategies that induce disagreement at the expense of compliance with the prompt.

\begin{table}[t]
    \centering
    \caption{Leader ablation on IFEval. Accuracy (\%) is reported
    for the Follower trained with an adaptive Leader
    (Stackelberg Follower) and with a fixed Leader
    (follower-only), evaluated on the same $541$ prompts and
    $834$ verifiable instructions. The higher value in each
    row is bolded. Category rows report loose instruction-level
    accuracy.}
    \label{tab:ifeval-lr2-follower-only}
    \small
    \setlength{\tabcolsep}{7pt}
    \renewcommand{\arraystretch}{1.08}
    \begin{tabular}{@{}lrr@{}}
        \toprule
        \textsc{Metric}
            & \textsc{Stackelberg Follower}
            & \textsc{Follower-only} \\
        \midrule
        Loose: prompt
            & \textbf{35.12} & 28.28 \\
        Loose: instruction
            & \textbf{46.04} & 38.49 \\
        Strict: prompt
            & \textbf{30.31} & 23.66 \\
        Strict: instruction
            & \textbf{41.25} & 33.93 \\
        \midrule
        \multicolumn{3}{@{}l}{
            \emph{Loose accuracy by instruction category}
        } \\
        Case
            & 22.47 & \textbf{33.71} \\
        Combination
            & \textbf{10.77} & 9.23 \\
        Detectable content
            & \textbf{50.94} & 41.51 \\
        Detectable format
            & \textbf{63.69} & 49.04 \\
        Keywords
            & \textbf{55.83} & 53.99 \\
        Language
            & \textbf{74.19} & 64.52 \\
        Length
            & \textbf{44.76} & 33.57 \\
        Punctuation
            & \textbf{19.70} & 16.67 \\
        Start/end
            & \textbf{58.21} & 28.36 \\
        \bottomrule
    \end{tabular}
\end{table}

\section{Example Prompts and Responses}
\label{app:example-responses}

We present selected high variance Leader and Follower comparisons to illustrate the response behaviors encountered during adversarial search. They illustrate unsupported elaboration, discrepancies between reasoning and final-answer correctness, and failures to preserve the requirements of a prompt. 

\subsection*{Example 1 (Unsupported elaboration)}

The Leader directly supplies the name given in the passage. The Follower reproduces that answer but adds a speculative expansion of the initial ``S.'' that the passage does not support.  This example illustrates how elaboration can introduce unsupported specificity into an otherwise sufficient answer. 

\begin{promptbox}[colback=promptgreen,colframe=promptgreenborder]{Prompt}
Quick Question: The chain's origin can be traced to the Dwarf Grill (now the Dwarf House), a restaurant opened by S. Truett Cathy, the chain's former chairman and CEO, in 1946. The restaurant is located in Hapeville, Georgia, a suburb of Atlanta, and is near the location of the now-demolished Ford Motor Company Atlanta Assembly Plant, for many years a source of many of the restaurant's patrons.

Question: who was the founder of chick fil a?
My answer:
\end{promptbox}

\begin{promptbox}[colback=promptred,colframe=promptredborder]{Leader Response}
The founder of Chick-fil-A is S. Truett Cathy.
\end{promptbox}

\begin{promptbox}[colback=promptultramarine,colframe=promptultramarineborder]{Follower Response}
To answer the question, let's look at the context provided:

"The chain's origin can be traced to the Dwarf Grill (now the Dwarf House), a restaurant opened by S. Truett Cathy, the chain's former chairman and CEO, in 1946."

From this sentence, we can see that S. Truett Cathy founded Chick-fil-A. The full name given is S. Truett Cathy, which likely stands for Stephen Truett Cathy. The word "Cathy" is often used as a family name in American English, so we can assume this refers to the founder.
\end{promptbox}

\subsection*{Example 2 (Correct final answer with flawed reasoning)}

Under the natural reading that each pack contains $30$ utensils split equally among knives, forks, and spoons, each pack contains $10$ spoons and the required number of packs is $5$. The Leader reaches this final answer through incorrect and internally inconsistent reasoning. The Follower is much more concise, but divides by $3$ and returns the incorrect answer $17$. This comparison exposes a tension between final-answer correctness and explanation quality.

\begin{promptbox}[colback=promptgreen,colframe=promptgreenborder]{Prompt}
John decides to buy utensils. They come in 30 packs with an equal number of knives, forks, and spoons. How many packs does he need to buy if he wants 50 spoons?
\end{promptbox}

\begin{promptbox}[colback=promptred,colframe=promptredborder]{Leader Response}
To find out how many packs John needs to buy, we start by figuring out how many teaspoons (knives and forks) are needed in total:

Number of spoons per pack = Number of packs x Number of spoons per pack

Since each pack contains an equal number of knives, forks, and spoons:
Teaspoons per pack = 3 spoons/pack

Now calculate the total teaspoons needed:
Total teaspoons needed = Number of spoons wanted x Number of spoons per pack
Total teaspoons needed = 50 spoons x 3 spoons/pack
Total teaspoons needed = 150 spoons

Each pack provides 3 spoons, but since there's no specific information about the distribution between knives and forks, let's assume he needs exactly 30 packs to get 90 spoons for simplicity.

Therefore, the number of packs John needs to buy is 150 spoons / 30 spoons/pack
Number of packs needed = 150 spoons / 30 spoons/pack
Number of packs needed = 5 packs

So, John needs to buy 5 packs of utensils to have 50 spoons.
\end{promptbox}

\begin{promptbox}[colback=promptultramarine,colframe=promptultramarineborder]{Follower Response}
He would need 50/3 = 16.67 packs, which rounds up to 17 packs. The answer is 17.
\end{promptbox}

\subsection*{Example 3 (Errors motivated by a rewrite)}

Both responses classify the original declarative sentence as exclamatory and replace ``over 100 million'' with ``100 million'', losing the stated lower bound. The Follower additionally inserts an exclamation mark and claims that its rewrite ``maintains the exclamation'', although the original sentence contains none. The rewrite makes the mistaken classification appear consistent with the revised text. 

\begin{promptbox}[colback=promptgreen,colframe=promptgreenborder]{Prompt}
Identify the type of the following sentence, rewrite it in a more concise form, and incorporate the budget of the movie as an additional constraint. Determine the sentence type and simplify: The Tom Cruise movie was a huge hit, with a budget of over 100 million.
\end{promptbox}

\begin{promptbox}[colback=promptred,colframe=promptredborder]{Leader Response}
Sentence Type: Exclamatory Sentence

Simplified Form: "The Tom Cruise movie had an incredible success with a staggering budget of 100 million."

Incorporating Budget Constraint:
"The Tom Cruise movie was a massive hit despite its extremely high budget of 100 million."
\end{promptbox}

\begin{promptbox}[colback=promptultramarine,colframe=promptultramarineborder]{Follower Response}
Type of Sentence: Exclamatory

Concise and Incorporating Budget Constraint:
"Tom Cruise's movie is a huge smash, with an astronomical budget of 100 million!"

This version maintains the exclamation, simplifies the language, and incorporates the budget detail.
\end{promptbox}

\subsection*{Example 4 (Loss of implicit alignment)}

The prompt refers to a sentence that is not supplied. The Leader acknowledges this missing input and offers a qualified suggestion. The Follower omits the qualification and substitutes scheduled punctuality for arriving just in time to prevent a negative consequence, which need not describe the same situation.

\begin{promptbox}[colback=promptgreen,colframe=promptgreenborder]{Prompt}
Given this sentence about punctuality, suggest a more appropriate and precise verb to convey the notion of arriving just in time to prevent a negative consequence.
\end{promptbox}

\begin{promptbox}[colback=promptred,colframe=promptredborder]{Leader Response}
The original sentence is not provided, but I can offer a suggestion based on common scenarios where punctuality might be crucial:

"If you arrive **on** time, you will avoid any negative consequences."

This phrase accurately conveys that being there exactly when needed prevents issues from arising.
\end{promptbox}

\begin{promptbox}[colback=promptultramarine,colframe=promptultramarineborder]{Follower Response}
"Arriving precisely at the scheduled moment"
\end{promptbox}

\section{Why General Preference Models Can Be More Prone to Reward Hacking}\label{sec:preference-model-hacking}

General preference models can represent intransitive preferences that scalar reward models cannot capture. This additional expressivity can, however, increase the number of annotations needed to estimate preferences accurately. We illustrate this tradeoff with a toy example in which the ground-truth preferences admit a scalar reward representation. In this setting, a reward model can pool information across comparisons, whereas an unrestricted preference estimator learns each pair separately. Fix a prompt $x$ and a response set $\mathcal{Y}={y_1,\ldots,y_K}$ with $K \ge 2$. Write
\begin{equation*}
    \mathcal{P}^*_{ij}:=\mathcal{P}^*(y_i\succ y_j\mid x),
\end{equation*}
which we assume to be antisymmetric. . For each of the $M=\binom{K}{2}$ unordered pairs, we observe $n$ annotations
\begin{equation*}
    Y_{ij}^{(\ell)}\sim\mathrm{Bernoulli}(\mathcal{P}^*_{ij}),
\qquad i<j,\quad \ell=1,\dots,n
\end{equation*}
independently across pairs and annotations, for a total budget of $nM$ observations. Define
\begin{equation*}
    \bar Y{ij}:=\frac1n\sum_{\ell=1}^nY_{ij}^{(\ell)}
\quad\text{for }i<j,
\qquad
\bar Y_{ji}:=1-\bar Y_{ij},
\qquad
\bar Y_{ii}:=\frac12.
\end{equation*}
To compare the two estimators under the same ground truth, suppose
\begin{equation*}
  \mcP^*_{ij}=\frac12+r_i-r_j,
\qquad
\sum_{i=1}^{K} r_i=0,
\qquad
|r_i|\leq\frac14
\quad\text{for all }i=1,\dots,K. 
\end{equation*}
where $r_i$ is the reward of response $y_i$. We measure estimation error in the entrywise supremum norm,
\begin{equation*}
\|\hat{\mathcal{P}}-\mathcal{P}^*\|_\infty
:=\max_{i,j=1,\dots,K}
\left|\hat{\mathcal{P}}_{ij}-\mathcal{P}^*_{ij}\right|.
\end{equation*}
Fix a confidence parameter $\delta\in(0,1)$.

\paragraph{General preference estimator.}
Consider the unbiased estimator $\hat{\mcP}^{\mathrm{pref}}_{ij}:=\bar Y_{ij}$. For each $i<j$, Hoeffding's inequality gives
\[
\PP\!\left(
    |\bar Y_{ij}-\mcP^*_{ij}|\geq\alpha
\right)
\leq 2\exp(-2n\alpha^2).
\]
Taking a union bound over the $M$ unordered pairs gives
\begin{equation}
\label{eq:preferencebound}
\|\hat{\mcP}^{\mathrm{pref}}-\mcP^*\|_\infty
\leq
\sqrt{\frac{\log(2M/\delta)}{2n}}
\end{equation}
with probability at least $1-\delta$. This bound does not require
the scalar reward representation assumed above.

\paragraph{Reward estimator.}
Now consider
\[
\hat r_i
:=\frac{1}{K}\sum_{j\neq i}
\left(\bar Y_{ij}-\frac12\right).
\]
Under the assumed reward representation, this estimator is unbiased:
\[
\EE[\hat r_i]
=\frac{1}{K}\sum_{j\neq i}(r_i-r_j)
=\frac{1}{K}\left(Kr_i-\sum_{j=1}^K r_j\right)
=r_i.
\]
We estimate preferences by
\[
\hat{\mcP}^{\mathrm{rew}}_{ij}
:=\left[\frac12+\hat r_i-\hat r_j\right]_{[0,1]},
\]
where $[z]_{[0,1]}:=\min\{1,\max\{0,z\}\}$ clips the estimate
to the probability interval. Consequently,
\begin{equation*}
    |\hat{\mcP}^{\mathrm{rew}}_{ij}-\mcP^*_{ij}|
\leq |(\hat r_i-r_i)-(\hat r_j-r_j)| 
\leq |\hat r_i-r_i|+|\hat r_j-r_j|.
\end{equation*}
It therefore suffices to control the reward estimation errors
uniformly. For each fixed $i$, we can write
\[
\hat r_i-r_i
=\frac{1}{nK}
\sum_{j\neq i}\sum_{\ell=1}^n
\left(Y_{ij}^{(\ell)}-\mcP^*_{ij}\right),
\]
where $Y_{ij}^{(\ell)}:=1-Y_{ji}^{(\ell)}$ for $j<i$. Hoeffding's inequality thus gives
\[
\PP\!\left(
    |\hat r_i-r_i|\geq\frac{\alpha}{2}
\right)
\leq
2\exp\!\left(-\frac{nK^2\alpha^2}{2(K-1)}\right).
\]
Taking a union bound over the $K$ responses, we obtain
\[
\PP\!\left(
    \max_{i\in[K]}|\hat r_i-r_i|\geq\frac{\alpha}{2}
\right)
\leq
2K\exp\!\left(-\frac{nK^2\alpha^2}{2(K-1)}\right).
\]
Setting the right-hand side equal to $\delta$ and applying
the bound on pairwise preference errors gives
\begin{equation}\label{eq:rewardbound}
\|\hat{\mcP}^{\mathrm{rew}}-\mcP^*\|_\infty
\leq
\sqrt{\frac{2(K-1)\log(2K/\delta)}{nK^2}}
\end{equation}
with probability at least $1-\delta$. 

Comparing Equations~\eqref{eq:preferencebound} and~\eqref{eq:rewardbound}, we find that, for fixed $\delta$, the uniform-error upper bound for the general preference estimator is larger by a factor of order $\sqrt{K}$. The reward estimator gains accuracy by pooling information across the $K-1$ comparisons involving each response, whereas the general preference estimator uses only the $n$ annotations for each pair. Thus, under the same annotation budget of $nM$, this example illustrates how general preference estimation can admit a larger upper bound on the errors that optimization may exploit.

\end{document}